\documentclass[12pt]{article}

\pdfoutput=1

\usepackage[T1]{fontenc}
\usepackage[utf8]{inputenc}
\usepackage{hyperref}      
\usepackage{url}            
\usepackage{booktabs}       
\usepackage{amsfonts}       
\usepackage{nicefrac}       
\usepackage{microtype}
\usepackage{wrapfig}
\usepackage{color}
\usepackage{booktabs}

\usepackage[numbers]{natbib}

\usepackage[preprint]{neurips_2019}

\usepackage{enumerate}
\usepackage{graphicx}
\usepackage{amsmath,amssymb,amsthm}
\usepackage{bigstrut}
\usepackage{float}
\usepackage[font=small]{caption}
\usepackage{mathtools}
\usepackage{subfigure}

\usepackage{algorithm}
\usepackage{algorithmic}

\newtheorem{theorem}{Theorem}

\newtheorem{lemma}[theorem]{Lemma}

\newtheorem{cor}[theorem]{Corollary}

\newtheorem*{theorem*}{Theorem}
\newtheorem*{lemma*}{Lemma}
\newtheorem*{cor*}{Corollary}

\DeclareMathOperator*{\maj}{maj}

\newcommand{\real}{\mathbb{R}}

\newcommand{\E}{\mathbb{E}}
\newcommand{\dracolite}{\textsc{Detox}}
\newcommand{\bulyan}{\textsc{Bulyan}}
\newcommand{\multikrum}{\textsc{Multi-krum}}
\newcommand{\signum}{\textsc{sign}SGD}
\newcommand{\eg}{\textit{e.g.}}
\newcommand{\ie}{\textit{i.e.}}
\newcommand{\mA}{\mathcal{A}}
\newcommand{\hier}{\textsc{Hier-Aggr}}
\newcommand{\ra}[1]{\renewcommand{\arraystretch}{#1}}

\title{\dracolite{}: A Redundancy-based Framework for Faster and More Robust Gradient Aggregation}
\author{
	Shashank Rajput\thanks{Authors contributed equally to this paper and are listed alphabetically.}\\
	University of Wisconsin-Madison\\
	\texttt{rajput3@wisc.edu} \\
	\And
	Hongyi Wang\footnotemark[1]\\
	University of Wisconsin-Madison\\
	\texttt{hongyiwang@cs.wisc.edu} \\
	\And
	Zachary Charles\\
	University of Wisconsin-Madison\\
	\texttt{zcharles@wisc.edu} \\
	\And
	Dimitris Papailiopoulos\\
	University of Wisconsin-Madison\\
	\texttt{dimitris@papail.io} \\
}

\begin{document}
	
	\maketitle
	
	\begin{abstract}
	To improve the resilience of distributed  training to worst-case, or Byzantine node failures, several recent approaches have replaced gradient averaging with robust aggregation methods.
	Such techniques can have high computational costs, often quadratic in the number of compute nodes, and only have limited robustness guarantees. Other methods have instead used redundancy to guarantee robustness, but can only tolerate limited number of Byzantine failures.
	In this work, we present \dracolite{}, a Byzantine-resilient distributed training framework that combines algorithmic redundancy with robust aggregation. \dracolite{} operates in two steps, a filtering step that uses limited redundancy to significantly reduce the effect of Byzantine nodes, and a hierarchical aggregation step that can be used in tandem with any state-of-the-art robust aggregation method. We show theoretically that this leads to a substantial increase in robustness, and has a per iteration runtime that can be nearly linear in the number of compute nodes.
	We provide extensive experiments over real distributed setups across a variety of large-scale machine learning tasks, showing that \dracolite{} leads to orders of magnitude accuracy and speedup improvements over many state-of-the-art Byzantine-resilient approaches.
	\end{abstract}
	
	\setlength{\abovedisplayskip}{3pt}
	\setlength{\belowdisplayskip}{3pt}
	
\section{Introduction}\label{Sec:Intro}
To scale the training of machine learning models, gradient computations can often be distributed across multiple compute nodes. 
After computing these local gradients, a parameter server then averages them, and updates a global model. As the scale of data and available compute power grows, so does the probability that some compute nodes output unreliable gradients. This can be due to power outages, faulty hardware, or communication failures, or due to security issues, such as the presence of an adversary governing the output of a compute node. 

Due to the difficulty in quantifying these different types of errors separately, we often model them as Byzantine failures. Such failures are assumed to be able to result in any output, adversarial or otherwise. Unfortunately, the presence of a single Byzantine compute node can result in arbitrarily bad global models when aggregating gradients via their average \cite{blanchard17}.


In the context of distributed training, there have generally been two distinct approaches to improve Byzantine robustness. The first replaces the gradient averaging step at the parameter server with a robust aggregation step, such as the geometric median and variants thereof \cite{blanchard17,chen2017distributed,xie18, Damaskinos2019aggregathor, DBLP:journals/corr/abs-1806-05358, yin2018byzantine}. The second approach instead assigns each node redundant gradients, and uses this redundancy to eliminate the effect of Byzantine failures \cite{chen2018draco,data2018data,yu2018lagrange}.

Both of the above approaches have their own limitations. 
For the first, robust aggregators are typically expensive to compute and scale super-linearly (in many cases quadratically \cite{mhamdi2018hidden,Damaskinos2019aggregathor}) with the number of compute nodes.
Moreover, such methods often come with limited theoretical guarantees of Byzantine robustness (\eg, only establishing convergence in the limit, or only guaranteeing that the output of the aggregator has positive inner product with the true gradient \cite{blanchard17,mhamdi2018hidden}) and often require strong assumptions, such as bounds on the dimension of the model being trained.
On the other hand, redundancy or coding-theoretic based approaches offer strong guarantees of perfect receovery for the aggregated gradients. However, such approaches, in the worst-case, require each node to compute $\Omega(q)$ times more gradients, where $q$ is the number of Byzantine machines \cite{chen2018draco}. This overhead is prohibitive in settings with a large number of Byzantine machines.

\begin{figure}[H]
	\centering
	\begin{minipage}{.65\textwidth}
		\centering
		\includegraphics[width=\textwidth]{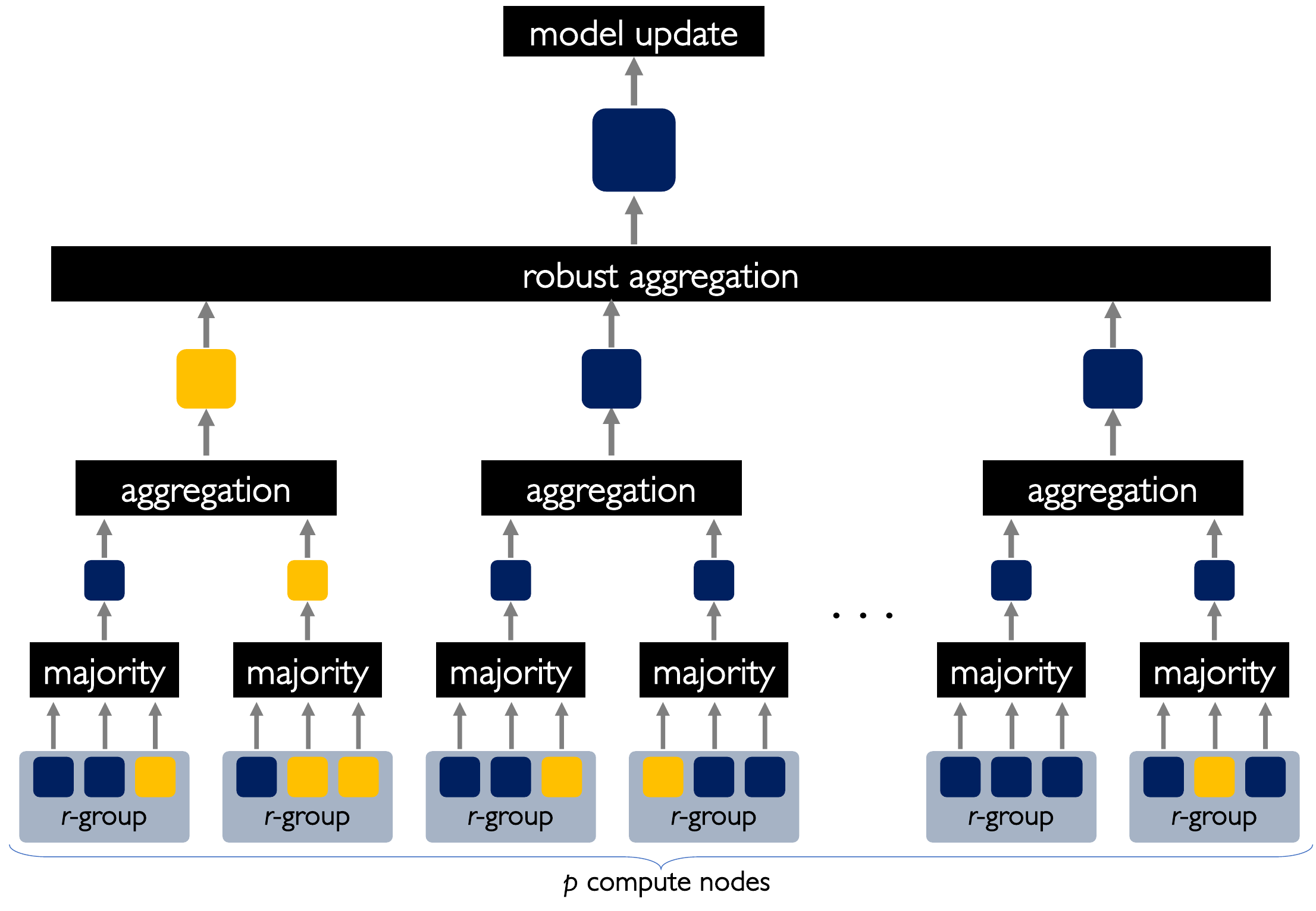}
		\caption{\dracolite{} is a hierarchical scheme for Byzantine gradient aggregation. In its first step, the parameter server partitions the compute nodes in groups and assigns each node to a group the same batch of data. After the nodes compute gradients with respect to this batch, the PS takes a majority vote of their outputs. This filters out a large fraction of the Byzantine gradients. In the second step, the parameter server partitions the filtered gradients in large groups, and applies a given aggregation method to each group. In the last step, the parameter server applies a robust aggregation method (\eg, geometric median) to the previous outputs. The final output is used to perform a gradient update step.} 
		\label{fig:detox}
	\end{minipage}%
	\hspace{0.03\textwidth}
	\begin{minipage}{.31\textwidth}
		\centering
		\includegraphics[width=\textwidth]{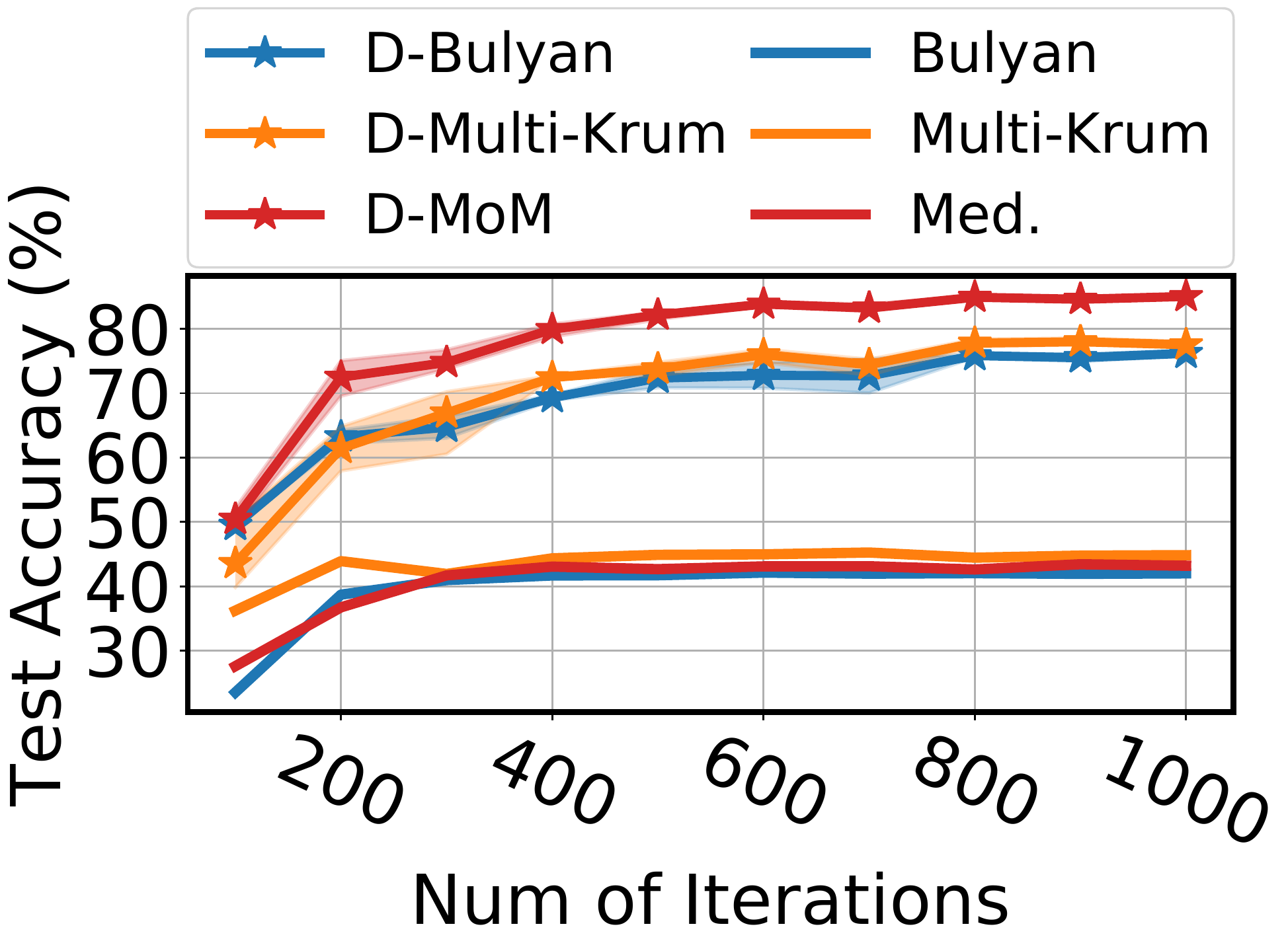}\\
		\includegraphics[width=\textwidth]{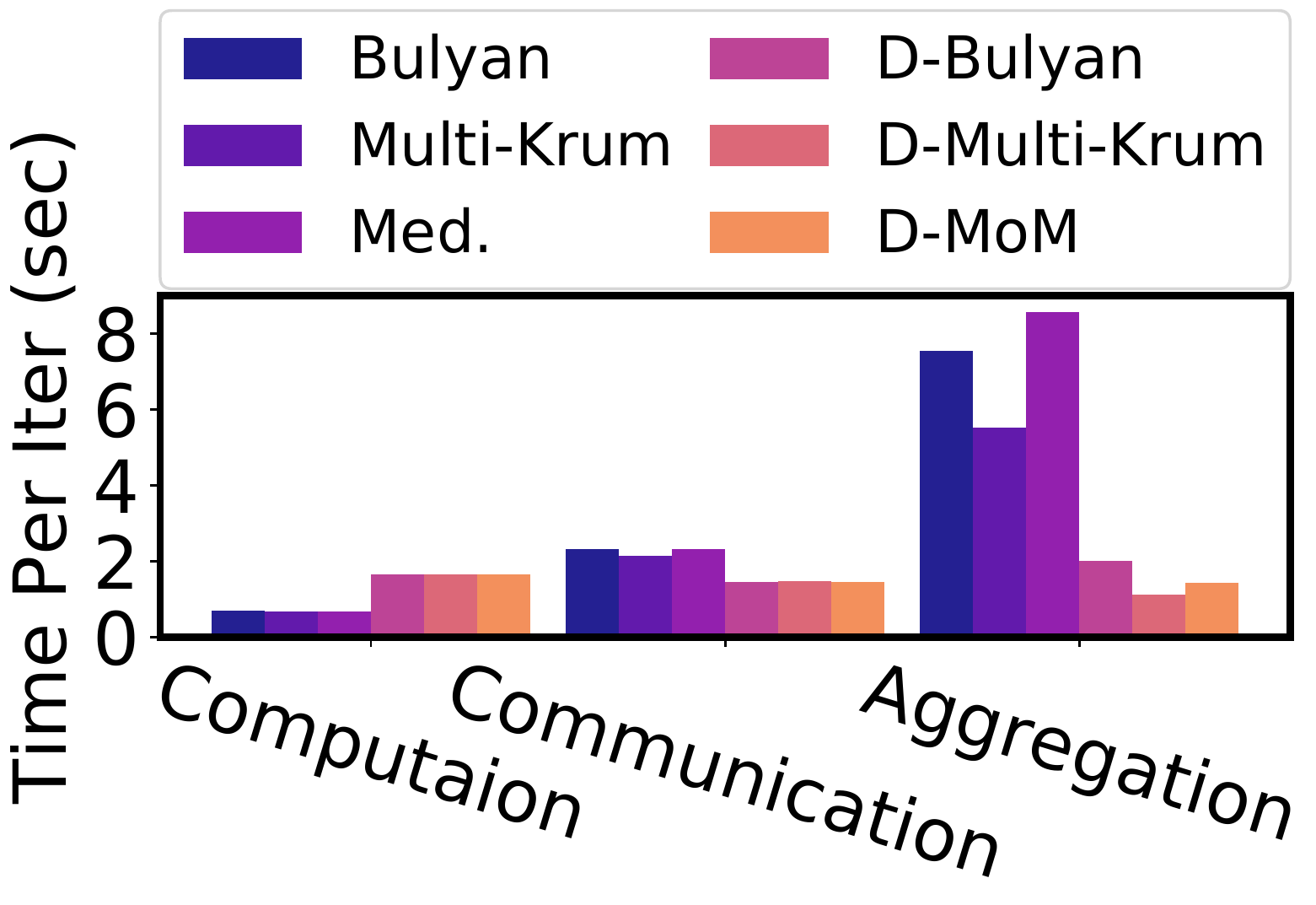}
		\caption{Top: convergence comparisons among various vanilla robust aggregation methods and the versions after deploying \dracolite{} under ``a little is enough" Byzantine attack \cite{baruch2019little}. Bottom: Per iteration runtime analysis of various methods. All results are for ResNet-18 trained on CIFAR-10. The prefix ``D-" stands for a robust aggregation method paired with \dracolite{}.} \label{fig:introLieConvResNet}
	\end{minipage}
\end{figure}

 \textbf{Our contributions.} In this work, we present \dracolite{}, a Byzantine-resilient distributed training framework that first uses computational redundancy to filter out almost all Byzantine gradients, and then performs a hierarchical robust aggregation method. \dracolite{} is scalable, flexible, and is designed to be used on top of any robust aggregation method to obtain improved robustness and efficiency. A high-level description of the hierarchical nature of \dracolite{} is given in Fig.~\ref{fig:detox}.
 
\dracolite{} proceeds in three steps. First the parameter server orders the compute nodes in groups of $r$ to compute the same gradients. While this step requires redundant computation at the node level, it will eventually allow for much faster computation at the PS level, as well as improved robustness.
After all compute nodes send their gradients to the PS, the PS takes the majority vote of each group of gradients.
We show that by setting $r$ to be logarithmic in the number of compute nodes, after the majority vote step only a constant number of Byzantine gradients are still present, even if the number of Byzantine nodes is a {\it constant fraction} of the total number of compute nodes.
\dracolite{} then performs hierarchical robust aggregation in two steps: First, it partitions the filtered gradients in a small number of groups, and aggregates them using simple techniques such as averaging. Second, it applies any robust aggregator (\eg, geometric median~\cite{chen2017distributed,yin2018byzantine}, Bulyan~\cite{mhamdi2018hidden}, Multi-Krum~\cite{Damaskinos2019aggregathor}, etc.) to the averaged gradients to further minimize the effect of any remaining traces of the original Byzantine gradients.

We prove that \dracolite{} can obtain {\it orders of magnitude} improved robustness guarantees compared to its competitors, and can achieve this at a nearly linear complexity in the number of compute nodes $p$, unlike methods like Bulyan~\cite{mhamdi2018hidden} that require run-time that is quadratic in $p$. 
We extensively test our method in real distributed setups and large-scale settings, showing that by combining \dracolite{} with previously proposed Byzantine robust methods, such as Multi-Krum, Bulyan, and coordinate-wise median, we increase the robustness and reduce the overall runtime of the algorithm. Moreover, we show that under strong Byzantine attacks, \dracolite{} can lead to almost a 40\% increase in accuracy over vanilla implementations of Byzantine-robust aggregation.
A brief performance comparison with some of the current state-of-the-art aggregators in shown in Fig.~\ref{fig:introLieConvResNet}.


\paragraph{Related work.} 
The topic of Byzantine fault tolerance has been extensively studied since the early 80s by Lamport et al. \cite{lamport1982byzantine}, and deals with worst-case, and/or adversarial failures, \eg, system crashes, power outages, software bugs, and adversarial agents that exploit security flaws. In the context of distributed
optimization, these failures are manifested through a subset of compute nodes returning to the master
flawed or adversarial updates. It is now well understood that first-order methods, such as gradient descent or mini-batch SGD, are not robust to Byzantine errors; even a single erroneous update can introduce arbitrary errors to the optimization variables.

Byzantine-tolerant ML has been extensively studied in recent years \cite{el2019sgd,xie2018zeno,xie2019fall,el2019fast, blanchard2017machine,chen2017distributed}, establishing that while average-based gradient methods are susceptible to adversarial nodes, median-based update methods can in some cases achieve better convergence, while being robust to some attacks. Although theoretical guarantees are provided in many works, the proposed algorithms in many cases only ensure a weak form of resilience against Byzantine failures, and often fail against strong Byzantine attacks \cite{mhamdi2018hidden}. A stronger form of Byzantine resilience is desirable for most of distributed machine learning applications. To the best of our knowledge, \textsc{Draco} \cite{chen2018draco} and \textsc{Bulyan} \cite{mhamdi2018hidden} are the only proposed methods that guarantee strong Byzantine resilience. However, as mentioned above, \textsc{Draco} requires heavy redundant computation from the compute nodes, while \textsc{Bulyan} requires heavy computation overhead on the parameter server end.

 We note that \cite{NIPS2018_7712} presents an alternative approach that does not fit easily under either category, but requires convexity of the underlying loss function. Finally, \cite{bernstein2018signsgd} examines the robustness of signSGD with a majority vote aggregation, but study a restricted Byzantine failure setup that only allows for a blind multiplicative adversary.

\section{Problem Setup}\label{section:redqe}
Our goal is to solve solve the following empirical risk minimization problem:
$$ \min_{w} F(w) := \frac{1}{n}\sum_{i=1}^n f_i(w)$$
where $w \in \real^d$ denotes the parameters of a model, and $f_i$ is the loss function on the $i$-th training sample.
 To approximately solve this problem, we often use mini-batch SGD. First, we initialize at some $w_0$. At iteration $t$, we sample $S_t$ uniformly at random from $\{1,\ldots, n\}$, and then update via
\begin{equation}\label{eq:mini_batch_sgd}
 w_{t+1} = w_t - \frac{\eta_t}{|S_t|}\sum_{i \in S_t} \nabla f_i(w_t),\end{equation}
where $S_t$ is a randomly selected subset of the $n$ data points.
To perform mini-batch SGD in a distributed manner, the global model $w_t$ is stored at a parameter server (PS) and updated according to (\ref{eq:mini_batch_sgd}), \ie, by using the mean of gradients that are evaluated at the compute nodes.

Let $p$ denote the total number of compute nodes. At each iteration $t$, during distributed mini-batch SGD, the PS broadcasts $w_t$ to each compute node. Each compute node is assigned $S_{i,t} \subseteq S_t$, and then evaluates the sum of gradients 
$$g_i = \sum_{j \in S_{i,t}}\nabla f_j(w_t).$$
The PS then updates the global model via
$$w_{t+1} = w_t-\frac{\eta_t}{p}\sum_{i=1}^p g_i.$$

We note that in our setup we assume that the parameter server is the owner of the data, and has access to the entire data set of size $n$.

\textbf{Distributed training with Byzantine nodes}~~We assume that a fixed subset $Q$ of size $q$ of the $p$ compute nodes are Byzantine. Let $\hat{g}_i$ be the output of node $i$. If $i$ is not Byzantine ($i \notin Q$), we say it is ``honest'', in which case its output $\hat{g}_i = g_i$ where $g_i$ is the true sum of gradients assigned to node $i$. If $i$ is Byzantine ($i \in Q$), its output $\hat{g}_i$ can be any $d$-dimensional vector. The PS receives $\{ \hat{g}_i\}_{i=1}^p$, and can then process these vectors to produce some approximation to the true gradient update in \eqref{eq:mini_batch_sgd}.

We make no assumptions on the Byzantine outputs. In particular, we allow adversaries with full information about $F$ and $w_t$, and that the byzantine compute nodes can collude. Let $\epsilon = q/p$ be the fraction of Byzantine nodes. We will assume $\epsilon < 1/2$ throughout.

\section{\dracolite{}: A Redundancy Framework to Filter most Byzantine Gradients}
We now describe \dracolite{}, a framework for Byzantine-resilient mini-batch SGD with $p$ nodes, $q$ of which are Byzantine. Let $b \geq p$ be the desired batch-size, and let $r$ be an odd integer. We refer to $r$ as the {\it redundancy ratio}. For simplicity, we will assume $r$ divides $p$ and that $p$ divides $b$. \dracolite{} can be directly extended to the setting where this does not hold. 

\dracolite{} first computes a random partition of $[p]$ in $p/r$ node groups $A_1,\ldots ,A_{p/r}$ each of size $r$. This will be fixed throughout. We then initialize at some $w_0$. For $t \geq 0$, we wish to compute some approximation to the gradient update in \eqref{eq:mini_batch_sgd}. To do so, we need a Byzantine-robust estimate of the true gradient. Fix $t$, and let us suppress the notation $t$ when possible. As in mini-batch SGD, let $S$ be a subset of $[n]$ of size $b$, with each element sampled uniformly at random from $[n]$. We then partition of $S$ in groups $S_{1},\ldots, S_{p/r}$ of size $br/p$.
For each $i \in A_j$, the PS assigns node $i$ the task of computing
\begin{equation}\label{eq:g_j_detox}
 g_{j} := \frac{1}{|S_{j}|}\sum_{k \in S_{j}} \nabla f_k(w) = \frac{p}{rb}\sum_{k \in S_{j}} \nabla f_k(w).
\end{equation}
If $i$ is an honest node, then its output is $\hat{g}_{i} = g_{j}$, while if $i$ is Byzantine, it outputs some $d$-dimensional $\hat{g}_{i}$. The $\hat{g}_{i}$ are then sent to the PS. The PS then computes 
$$z_{j} := \maj(\{\hat{g}_{i} | i \in A_j\}),$$ 
where $\maj$ denotes the majority vote. If there is no majority, we set $z_{j} = 0$. We will refer to $z_j$ as the ``vote'' of group $j$.

Since some of these votes are still Byzantine, we must do some robust aggregation of the vote. We employ a hierarchical robust aggregation process \hier{}, which uses two user-specified aggregation methods $\mA_0$ and $\mA_1$. First, the votes are partitioned in to $k$ groups. Let $\hat{z}_1,\ldots, \hat{z}_k$ denote the output of $\mA_0$ on each group. The PS then computes $\hat{G} = \mA_1(\hat{z}_1,\ldots, \hat{z}_k)$ and updates the model via $w = w - \eta \hat{G}$. This hierarchical aggregation resembles a median of means approach on the votes \cite{minsker2015geometric}, and has the benefit of improved robustness and efficiency. We discuss this in further detail in Section \ref{sec:speed_robust}.

A description of \dracolite{} is given in Algorithm \ref{alg:1}.



	\begin{algorithm}
	    \caption{\dracolite{}: Algorithm to be performed at the parameter server}
	    \begin{algorithmic}[1]\label{alg:1}
	    \INPUT Batch size $b$, redundancy ratio $r$, compute nodes $1,\ldots, p$, step sizes $\{\eta_t\}_{t \geq 0}$.
	    \STATE Randomly partition $[p]$ in  ``node groups'' $\{A_j | 1 \leq j \leq p/r\}$ of size $r$.
	    \FOR{$t=0$ {\bfseries to} $T$}
	        \STATE Draw $S_t$ of size $b$ randomly from $[n]$.
	        \STATE Partition $S_t$ in to groups $\{S_{t,j} | 1 \leq j \leq p/r\}$ of size $rb/p$.
	        \STATE For each $j \in [p/r], i \in A_j$, push $w_t$ and $S_{t,j}$ to compute node $i$.
	        \STATE Receive the (potentially Byzantine) $p$ gradients $\hat{g}_{t,i}$ from each node. 
	        \STATE Let $z_{t,j} := \maj(\{ \hat{g}_{t,i} | i \in A_j\})$, and 0 if no majority exists. \hfill \%Filtering step
	        \STATE Set $\hat{G}_t = \hier{}(\{z_{t,1},\ldots, z_{t,p/r}\})$. \hfill \%Hierarchical aggregation
	        \STATE Set $w_{t+1} = w_t - \eta \hat{G}_t$. \hfill \%Gradient update
	    \ENDFOR
	  \end{algorithmic}
	\end{algorithm}




	\begin{algorithm}
	    \caption{\hier{}: Hierarchical aggregation}
	    \begin{algorithmic}[1]\label{alg:2}
	    \INPUT Aggregators $\mA_0,\mA_1$, votes $\{z_1,\dots,z_{p/r}\}$, vote group size $k$.
	    \STATE Let $\hat{p}:=p/r$.
	    \STATE Randomly partition $\{z_1,\dots,z_{\hat{p}}\}$ in to ``vote groups'' $\{Z_j | 1 \leq j \leq \hat{p}/k\}$ of size $k$.
	    \STATE For each vote group $Z_j$, calculate $\hat{z}_j = \mA_0(Z_j)$.
	    \STATE Return $\mA_1(\{\hat{z}_1,\dots, \hat{z}_{\hat{p}/k}\})$.
	  \end{algorithmic}
	\end{algorithm}


\subsection{Filtering out Almost Every Byzantine Node}\label{section:reduceByzantine}

	We now show that \dracolite{} filters out the vast majority of Byzantine gradients. Fix the iteration $t$. Recall that all honest nodes in a node group $A_j$ send $\hat{g}_j = g_j$ as in \eqref{eq:g_j_detox} to the PS. If $A_j$ has more honest nodes than Byzantine nodes then $z_j = g_j$ and we say $z_j$ is honest. If not, then $z_j$ may not equal $g_j$ in which case $z_j$ is a Byzantine vote. Let $X_j$ be the indicator variable for whether block $A_j$ has more Byzantine nodes than honest nodes, and let $\hat{q} = \sum_{j} X_j$. This is the number of Byzantine votes. By filtering, \dracolite{} goes from a Byzantine compute node ratioof $\epsilon = q/p$ to a Byzantine vote ratio of $\hat{\epsilon} = \hat{q}/\hat{p}$ where $\hat{p} = p/r$. 

	We first show that $\E[\hat{q}]$ decreases {\it exponentially} with $r$, while $\hat{p}$ only decreases linearly with $r$. That is, by incurring a constant factor loss in compute resources, we gain an exponential improvement in the reduction of byzantine nodes. 
	 Thus, even small $r$ can drastically reduce the Byzantine ratio of votes. This observation will allow us to instead use robust aggregation methods on the $z_j$, \ie, the votes, greatly improving our Byzantine robustness. We have the following theorem about $\E[\hat{q}]$. All proofs can be found in the appendix. Note that throughout, we did not focus on optimizing constants.


	\begin{theorem}\label{lemma:2}
	There is a universal constant $c$ such that if the fraction of Byzantine nodes is $\epsilon < c$, then the effective number of Byzantine votes after filtering becomes $$\E[\hat{q}]=\mathcal{O}\left(\epsilon^{(r-1)/2}q/r\right).$$
	\end{theorem}
We now wish to use this to derive high probability bounds on $\hat{q}$.
While the variables $X_i$ are not independent, they are negatively correlated. By using a version of Hoeffding's inequality for weakly dependent variables,
we can show that if the redundancy is logarithmic, \ie, $r \approx \log(q)$, then with high probability the number of effective byzantine votes drops to a constant, \ie, $\hat{q} = \mathcal{O}(1)$.
\begin{cor}\label{lemma:5}There is a constant $c$ such that if and $\epsilon \leq c$ and $r\geq3 + 2\log_2(q) $ then for any $\delta \in (0,\frac{1}{2})$, with probability at least $1-\delta$, we have that $\hat{q} \leq 1+2\log(1/\delta)$.
\end{cor}
In the next section, we exploit this dramatic reduction of Byzantine votes to derive strong robustness guarantees for \dracolite{}.

\section{\dracolite{} Improves the Speed and Robustness of Robust Estimators}\label{sec:speed_robust}

	Using the results of the previous section, if we set the redundancy ratio to $r \approx \log(q)$, the filtering stage of \dracolite{} reduces the number of Byzantine votes $\hat{q}$ to roughly a constant. While we could apply some robust aggregator $\mA$ directly to the output votes of the filtering stage, such methods often scale poorly with the number of votes $\hat{p}$. By instead applying \hier, we greatly improve efficiency and robustness. Recall that in \hier, we partition the votes into $k$ ``vote groups'', apply some $\mA_0$ to each group, and apply some $\mA_1$ to the $k$ outputs of $\mA_0$. We analyze the case where $k$ is roughly constant, $\mA_0$ computes the mean of its inputs, and $\mA_1$ is a robust aggregator. In this case, \hier~is analogous to the Median of Means (MoM) method from robust statistics \cite{minsker2015geometric}.

	\paragraph{Improved speed.}

	Suppose that without redundancy, the time required for the compute nodes to finish is $T$. Applying Krum \cite{blanchard17}, Multi-Krum \cite{Damaskinos2019aggregathor}, and Bulyan \cite{mhamdi2018hidden} to their $p$ outputs requires $\mathcal{O}(p^2d)$ operations, so their overall runtime is $\mathcal{O}(T+p^2d)$. In \dracolite{}, the compute nodes require $r$ times more computation to evaluate redundant gradients. If $r \approx \log(q)$, this can be done in $\mathcal{O}(\ln(q)T)$. With \hier~as above, \dracolite{} performs three major operations: (1) majority voting, (2) mean computation of the $k$ vote groups and (3) robust aggregation of the these $k$ means using $\mA_1$. (1) and (2) require $\mathcal{O}(pd)$ time. For practical $\mA_1$ aggregators, including Multi-Krum and Bulyan, (3) requires $\mathcal{O}(k^2d)$ time. Since $k\ll p$, \dracolite{} has runtime $\mathcal{O}(\ln(q)T+pd)$. If $T = \mathcal{O}(d)$ (which generally holds for gradient computations), Krum, Multi-Krum, and Bulyan require $\mathcal{O}(p^2d)$ time, but \dracolite{} only requires $\mathcal{O}(pd)$ time. Thus, \dracolite{} can lead to significant speedups, especially when the number of workers is large.

	\paragraph{Improved robustness.}

	To analyze robustness, we first need some distributional assumptions. At any given iteration, let $G$ denote the full gradient of $F(w)$. Throughout this section, we assume that the gradient of each sample is drawn from a distribution $\mathcal{D}$ on $\mathbb{R}^d$ with mean $G$ and variance  $\sigma^2$. In \dracolite{}, the ``honest'' votes $z_i$ will also have mean $G$, but their variance will be $\sigma^2p/rb$. This is because each honest compute node gets a sample of size $rb/p$, so its variance is reduced by a factor of $rb/p$.

	Suppose $\hat{G}$ is some approximation to the true gradient $G$. We say that
	$\hat{G}$ is a $\Delta$-inexact gradient oracle for $G$ if $\|\hat{G} - G\|\leq \Delta$.
	\cite{DBLP:journals/corr/abs-1806-05358} shows that access to a $\Delta$-inexact gradient oracle is sufficient to upper bound the error of a model $\hat{w}$ produced by performing gradient updates with $\hat{G}$. To bound the robustness of an aggegator, it suffices to bound $\Delta$. Under the distributional assumptions above, we will derive bounds on $\Delta$ for the hierarchical aggregator $\mA$ with different base aggregators $\mA_1$.

	We will analyze \dracolite{} when $\mA_0$ computes the mean of the vote groups, and $\mA_1$ is geometric median, coordinate-wise median, or $\alpha$-trimmed mean \cite{yin2018byzantine}. We will denote the approximation $\hat{G}$ to $G$ computed by \dracolite{}~in these three instances by $\hat{G}_1, \hat{G}_2$ and $\hat{G}_3$, respectively. Using the proof techniques in \cite{minsker2015geometric}, we get the following.

	\begin{theorem}\label{prop:MoM}
	Assume $r\geq 3 + 2 \log_2(q)$ and $\epsilon \leq c$ where $c$ is the constant from Corollary \ref{lemma:5}. There are constants $c_1, c_2, c_3$ such that for all $\delta \in (0,1/2)$, with probability at least $1-2\delta$:
	\begin{enumerate}
		\item If $k=128\ln(1/\delta)$, then $\hat{G}_1$ is a $c_1\sigma\sqrt{\frac{\ln(1/\delta)}{b}}$-inexact gradient oracle.
		\item If $k=128\ln(d/\delta)$, then $\hat{G}_2$ is a $c_2\sigma\sqrt{\frac{\ln(d/\delta)}{b}}$-inexact gradient oracle.
		\item If $k=128\ln(d/\delta)$ and $\alpha = \frac{1}{4}$, then $\hat{G}_3$ is a $c_3\sigma\sqrt{\frac{\ln(d/\delta)}{b}}$-inexact gradient oracle.
	\end{enumerate}
	\end{theorem}	

	The above theorem has three important implications. First, we can derive robustness guarantees for \dracolite{} that are virtually independent of the Byzantine ratio $\epsilon$. Second, even when there are no Byzantine machines, it is known that no aggregator can achieve $\Delta = o(\sigma/\sqrt{b})$~\cite{lugosi2019sub}, and because we achieve $\Delta=\tilde{O}(\sigma/\sqrt{b})$, we cannot expect to get an order of better robustness by any other aggregator. Third, other than a logarithmic dependence on $q$, there is no dependence on the number of nodes $p$. Even as $p$ and $q$ increase, we still maintain roughly the same robustness guarantees.

	By comparison, the robustness guarantees of Krum and Geometric Median applied directly to the compute nodes worsens as as $p$ increases \cite{blanchard2017machine,xie18}. Similarly, \cite{yin2018byzantine} show if we apply coordinate-wise median to $p$ nodes, each of which are assigned $b/p$ gradients, we get a $\Delta$-inexact gradient oracle where $\Delta = \mathcal{O}(\sigma\sqrt{\epsilon p/b}+\sigma\sqrt{d/b})$.
	If $\epsilon$ is constant and $p$ is comparable to $b$, then this is roughly $\sigma$, whereas \dracolite{} can produce a $\Delta$-inexact gradient oracle for $\Delta = \tilde{\mathcal{O}}(\sigma/\sqrt{b})$. Thus, the robustness of \dracolite{} can scale much better with the number of nodes than naive robust aggregation of gradients.

\section{Experiments}\label{Sec:Experiment}
In this section we present an experimental study on pairing \dracolite{} with a set of previously proposed robust aggregation methods, including \multikrum{} \cite{blanchard2017machine}, \bulyan{} \cite{mhamdi2018hidden}, coordinate-wise median \cite{DBLP:journals/corr/abs-1806-05358}. We also incorporate \dracolite{} with a recently proposed Byzantine resilience distributed training method, \signum{} with majority vote \cite{bernstein2018signsgd}. We conduct extensive experiments on the scalability and robustness of these Byzantine resilient methods, and the improvements gained when pairing them with \dracolite{}. All our experiments are deployed on real distributed clusters under various Byzantine attack models. 
Our implementation is publicly available for reproducibility at \footnote{\href{https://github.com/hwang595/DETOX}{https://github.com/hwang595/DETOX}}.

The main findings are as follows: 
1) Applying \dracolite{} leads to significant speedups, \eg, up to an order of magnitude end-to-end training speedup is observed;
2) in defending against state-of-the-art Byzantine attacks, \dracolite{} leads to significant Byzantine-resilience, \eg, applying \bulyan{} on top of \dracolite{} improves the test-set prediction accuracy from 11\% to ~60\% when training VGG13-BN on CIFAR-100 under the ``a little is enough" (ALIE) \cite{baruch2019little} Byzantine attack. Moreover, incorporating \signum{} with \dracolite{} improves the test-set prediction accuracy from $34.92\%$ to $78.75\%$ when defending against a \textit{constatnt Byzantine attck} for ResNet-18 trained on CIFAR-10.

\subsection{Experimental Setup} 
We implemented vanilla versions of the aforementioned Byzantine resilient methods, as well as versions of these methods pairing with \dracolite{}, in PyTorch \cite{paszke2017automatic} with MPI4py \cite{MPI4PY}. Our experimental comparisons are deployed on a cluster of $46$ m5.2xlarge instances on Amazon EC2, where 1 node serves as the PS and the remaining $p=45$ nodes are compute nodes. In all following experiments, we set the number of Byzantine nodes to be $q=5$.

In each iteration of the vanilla Byzantine resilient methods, each compute node evaluates $\frac{b}{p}=32$ gradients sampled from its partition of data while in \dracolite{} each compute node evaluates $r$ times more gradients where $r = 3$, so $\frac{rb}{p} = 96$. The average of these locally computed gradients is then sent back to the PS. After receiving all gradient summations from the compute nodes, the PS applies either vanilla Byzantine resilient methods or their \dracolite{} paired variants. 
\subsection{Implementation of \dracolite{}}
\label{seubsec:implementation-of-detox}
We emphasize that \dracolite{} is not simply a new robust aggregation technique. It is instead a general Byzantine-resilient distributed training framework, and any robust aggregation method can be immediately implemented on top of it to increase its Byzantine-resilience and scalability. Note that after the majority voting stage on the PS one has a wide range of choices for $\mA_0$ and $\mA_1$. 
In our implementations, we had the following setups: 1) $\mA_0=$ Mean, $\mA_1=$ Coordinate-size Median, 2) $\mA_0=$ \multikrum{}, $\mA_1=$ Mean, 3) $\mA_0=$ \bulyan{}, $\mA_1=$ Mean, and 4) $\mA_0=$coordinate-wise majority vote, $\mA_1=$coordinate-wise majority vote (designed specifically for pairing \dracolite{} with \signum{}). We tried $\mA_0=$ Mean and $\mA_1=$ \multikrum{}/\bulyan{} but we found that setups 2) and 3) had better resilience than these choices.
More details on the implementation and system-level optimizations that we performed can be found in the Appendix \ref{appendix:implem}.

\paragraph{Byzantine attack models}
We consider two Byzantine attack models for pairing \multikrum{}, \bulyan{}, and coordinate-wise median with \dracolite{}. First, we consider the \textit{``reversed gradient"} attack, where adversarial nodes that were supposed to send ${\bf g}$ to the PS instead send  $-c {\bf g}$, for some $c>0$. 

The second Byzantine attack model we study is the recently proposed ALIE \cite{baruch2019little} attack, where the Byzantine compute nodes collude and use their locally calculated gradients to estimate the mean and standard deviation of the entire set of gradients among all other compute nodes. The Byzantine nodes then use the estimated mean and variance to manipulate the gradient they send back to the PS. To be more specific, Byzantine nodes will send $\hat \mu + z \cdot \hat \sigma$ where $\hat \mu$ and $\hat \sigma$ are the estimated mean and standard deviation by Byzantine nodes and $z$ is a hyper-parameter which was tuned empirically in \cite{baruch2019little}. 

Then, to compare the resilience of the vanilla \signum{} and the one paired with \dracolite{}, we will consider a simple attack, \ie, \textit{constant Byzantine attack}. In \textit{constant Byzantine attack}, Byzantine compute nodes simply send a constant gradient matrix with dimension equal to that of the true gradient where all elements equals to $-1$. Under this attack, and specifically for \signum{}, the Byzantine gradients will mislead model updates towards wrong directions and corrupt the final model trained via \signum{}. 

\paragraph{Datasets and models} We conducted our experiments over ResNet-18 \cite{he2016deep} on CIFAR-10 and VGG13-BN \cite{simonyan2014very} on CIFAR-100. For each dataset, we use data augmentation (random crops, and flips) and normalize each individual image. Moreover, we tune the learning rate scheduling process and use the constant momentum at $0.9$ in running all experiments. The details of parameter tuning and dataset normalization are reported in the Appendix \ref{appendix:paramTun}.

\subsection{Results}
\paragraph{Scalability} 
We report a per-iteration runtime analysis of the aforementioned robust aggregations and their \dracolite{} paired variants on both CIFAR-10 over ResNet-18 and CIFAR-100 over VGG-13. The results on  ResNet-18 and VGG13-BN are shown in Figure \ref{fig:introLieConvResNet} and \ref{fig:breakdownAnalysisVGG} respectively. 

We observe that although \dracolite{} requires slightly more compute time per iteration, due to its algorithmic redundancy, it largely reduces the PS computation cost during the aggregation stage, which matches our theoretical analysis. Surprisingly, we observe that by applying \dracolite{}, the communication costs decrease. This is because the variance of computation time among compute nodes increases with heavier computational redundancy. Therefore, after applying \dracolite{}, compute nodes tend not to send their gradients to the PS at the same time, which mitigates a potential network bandwidth congestion. In a nutshell, applying \dracolite{} can lead to up to 3$\times$ per-iteration speedup.
\begin{figure}[H]
	\centering
	\includegraphics[width=0.4\textwidth]{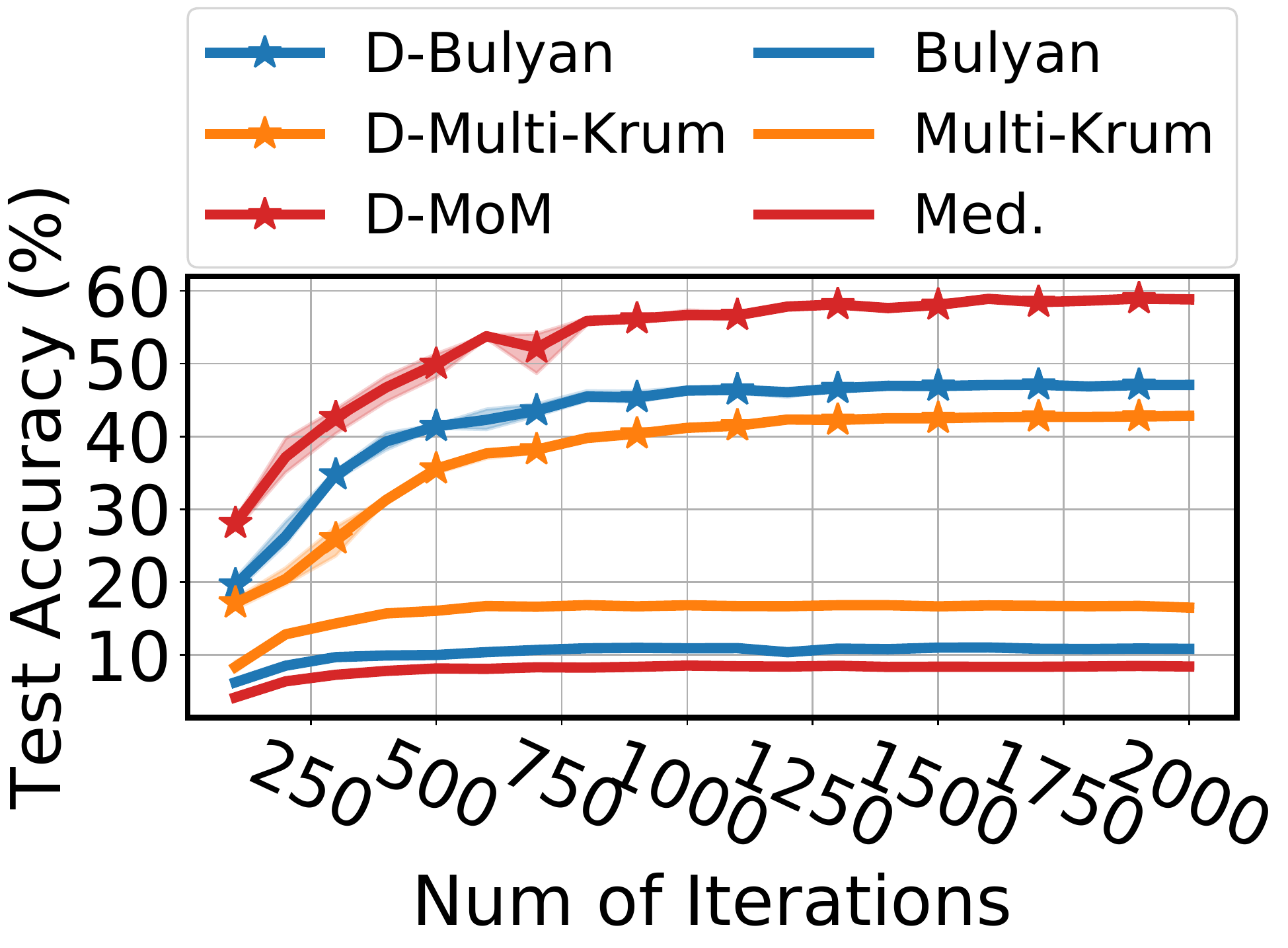}
	\includegraphics[width=0.44\textwidth]{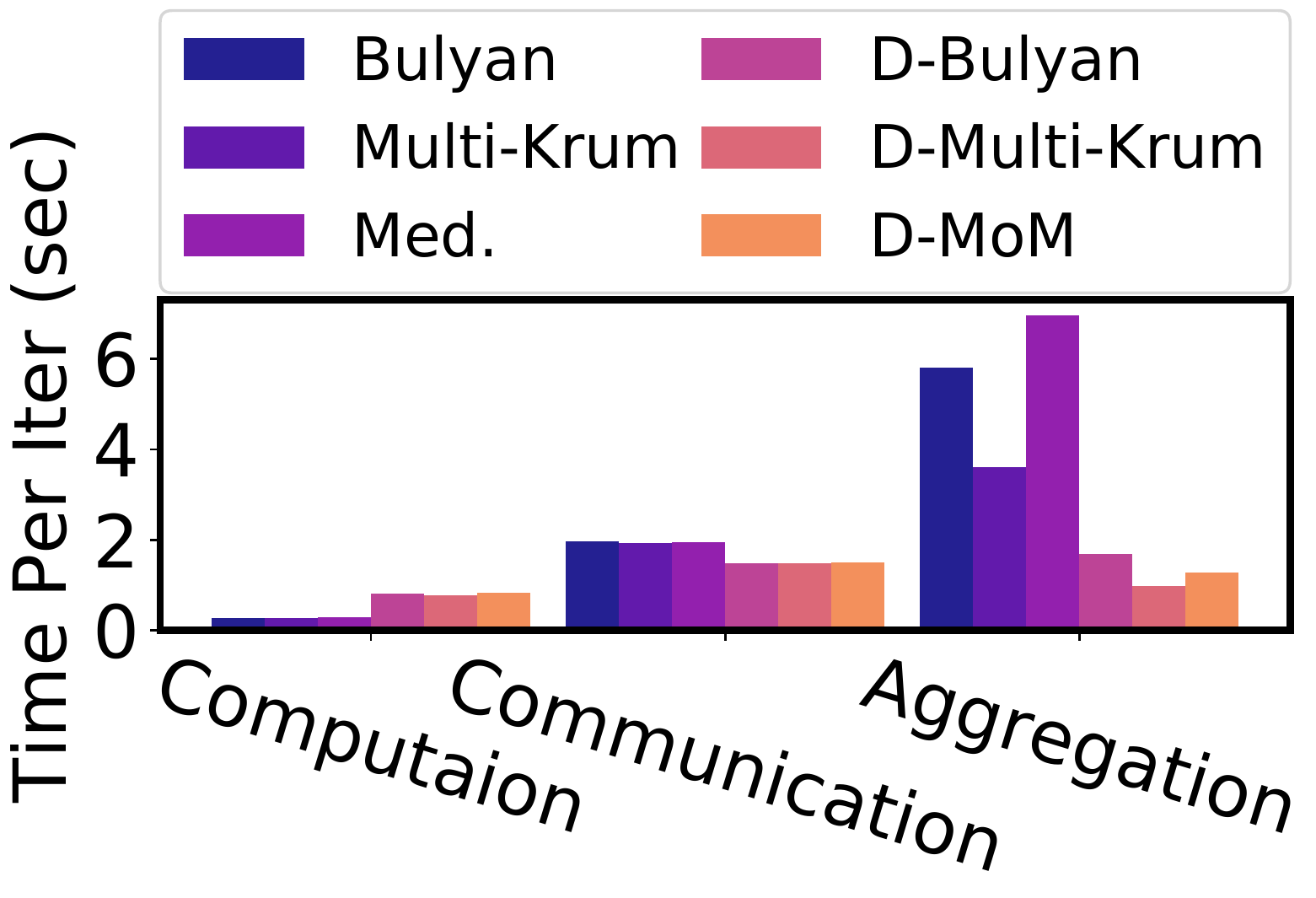}
	\caption{Left: Convergence performance of various robust aggregation methods over ALIE attack. Right: Per iteration runtime analysis of various robust aggregation methods. Results of VGG13-BN on CIFAR-100}\label{fig:breakdownAnalysisVGG}
\end{figure}

\begin{figure*}[ht]
	\centering      
	\subfigure[ResNet-18, \multikrum{}]{\includegraphics[width=0.31\textwidth]{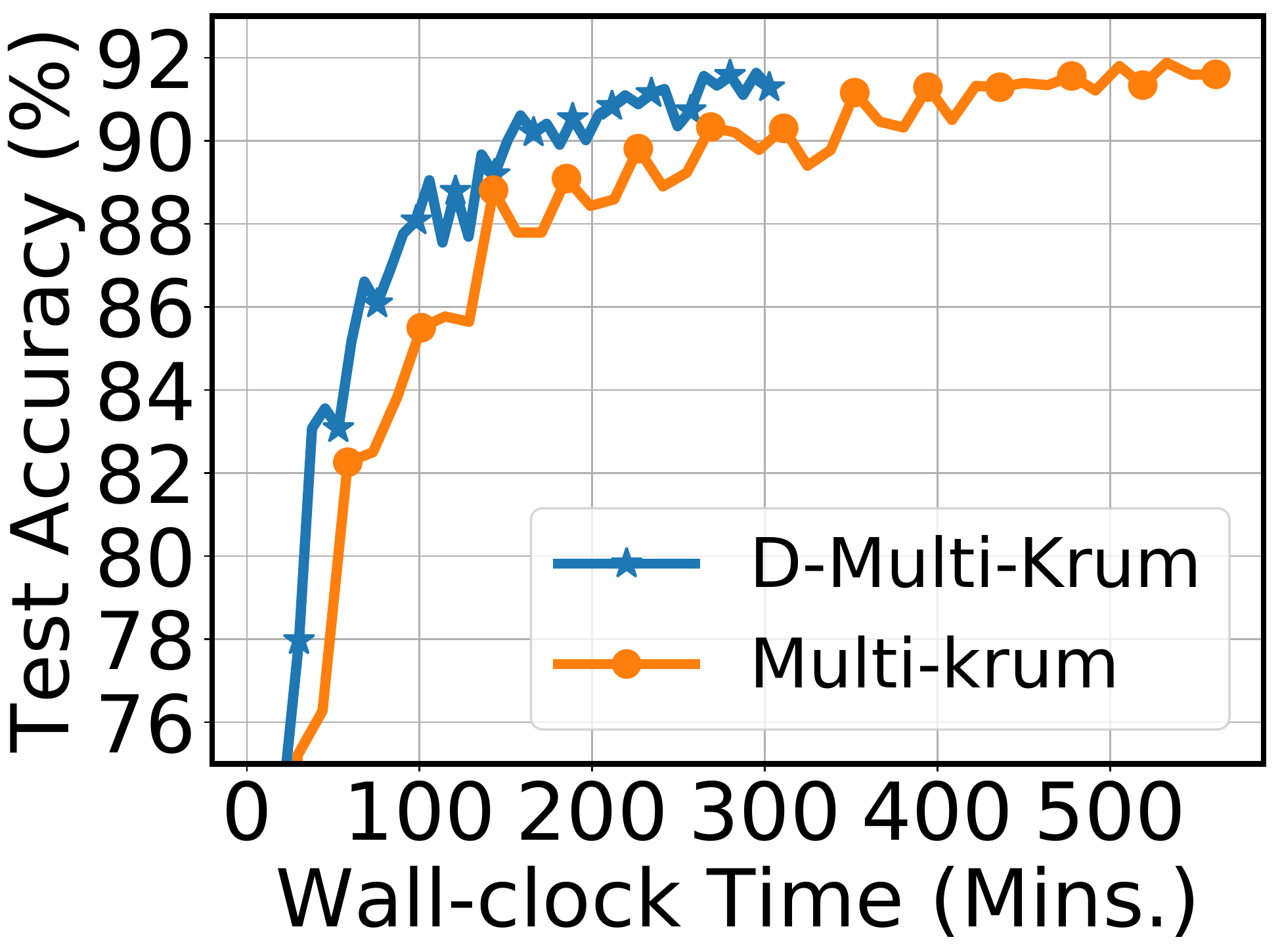}}
	\subfigure[ResNet-18, \bulyan{}]{\includegraphics[width=0.31\textwidth]{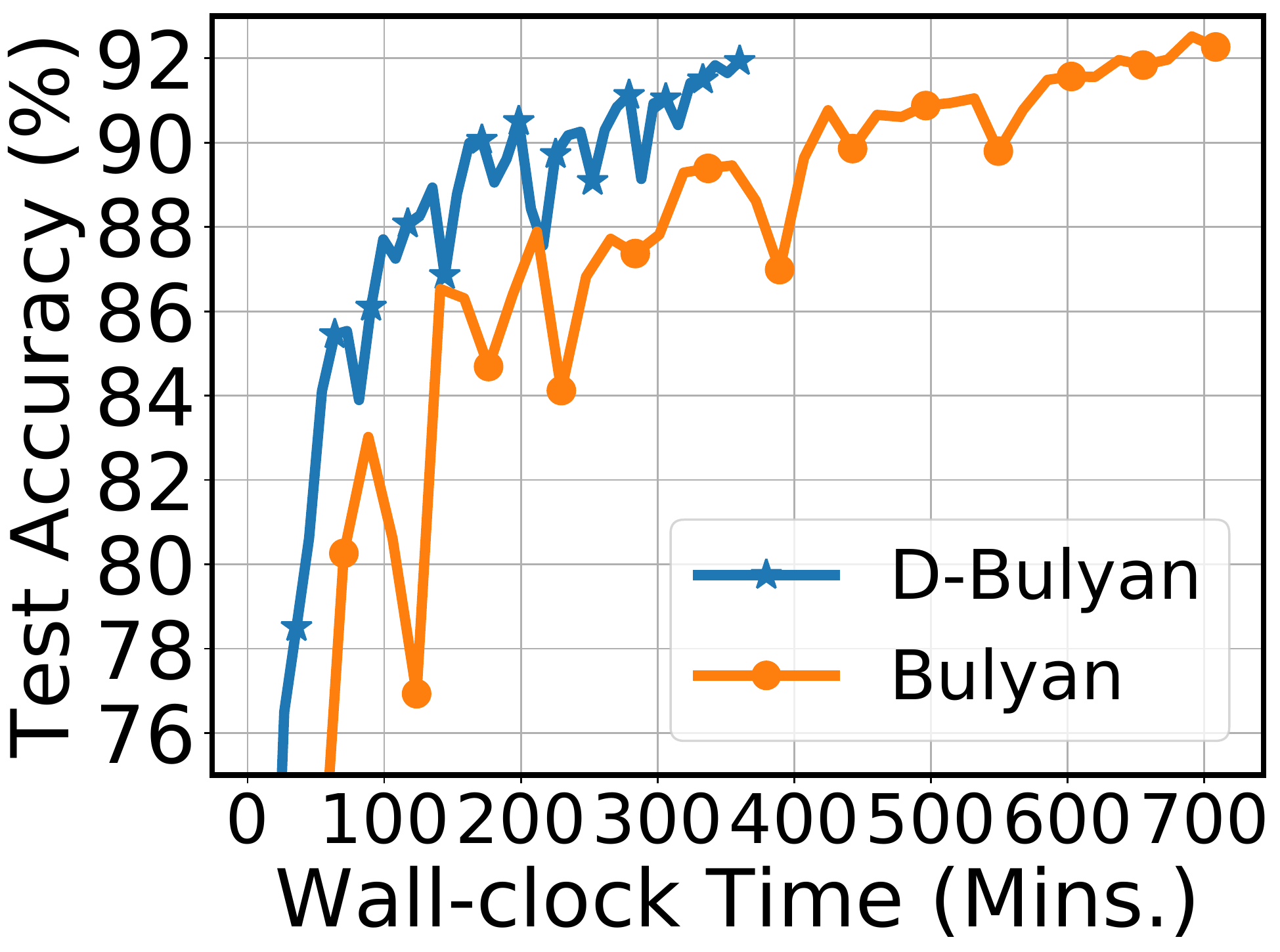}}
	\subfigure[ResNet-18, \textit{Coord-Median}]{\includegraphics[width=0.31\textwidth]{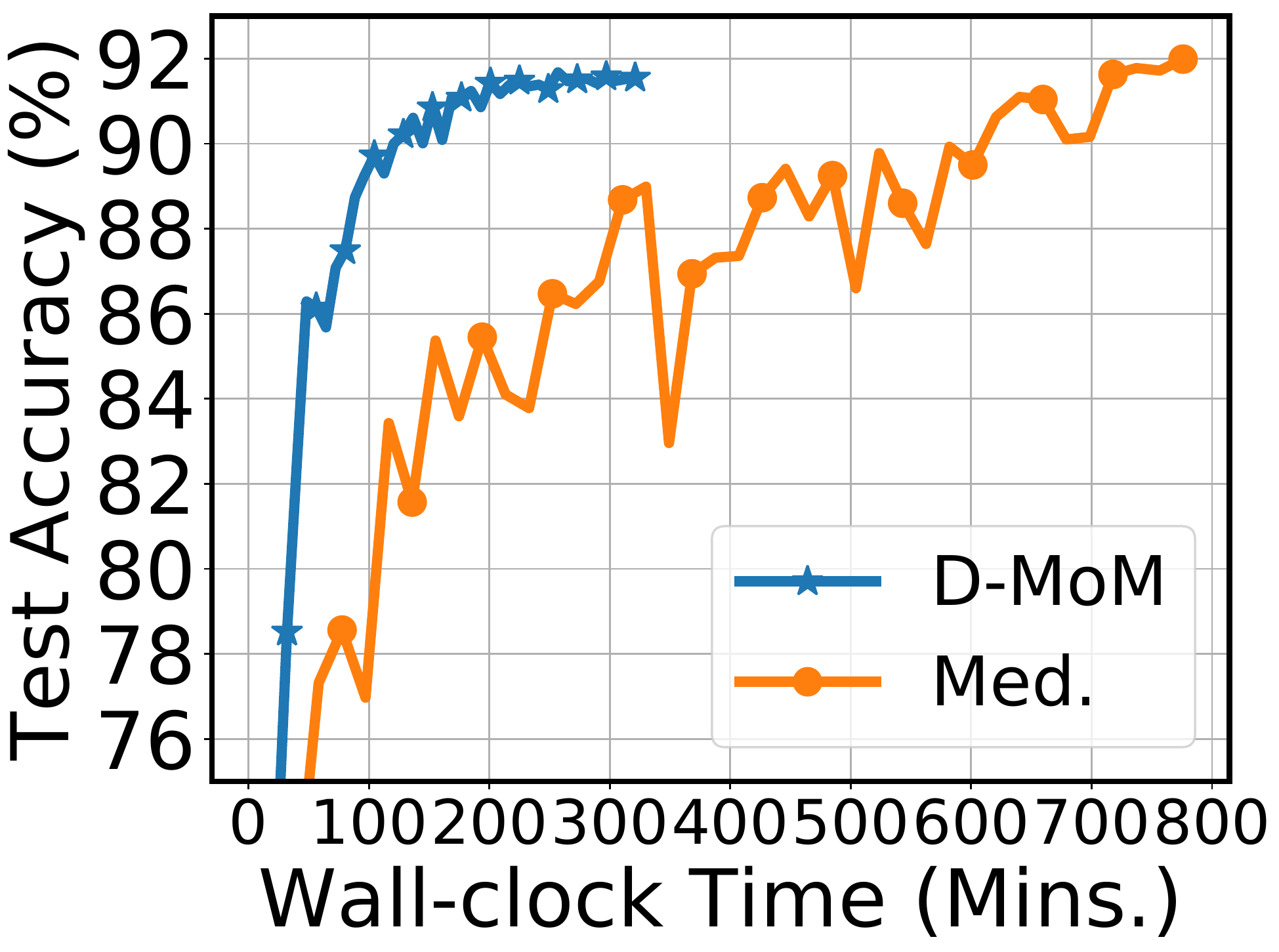}}
	\\
	\subfigure[VGG13-BN, \multikrum{}]{\includegraphics[width=0.31\textwidth]{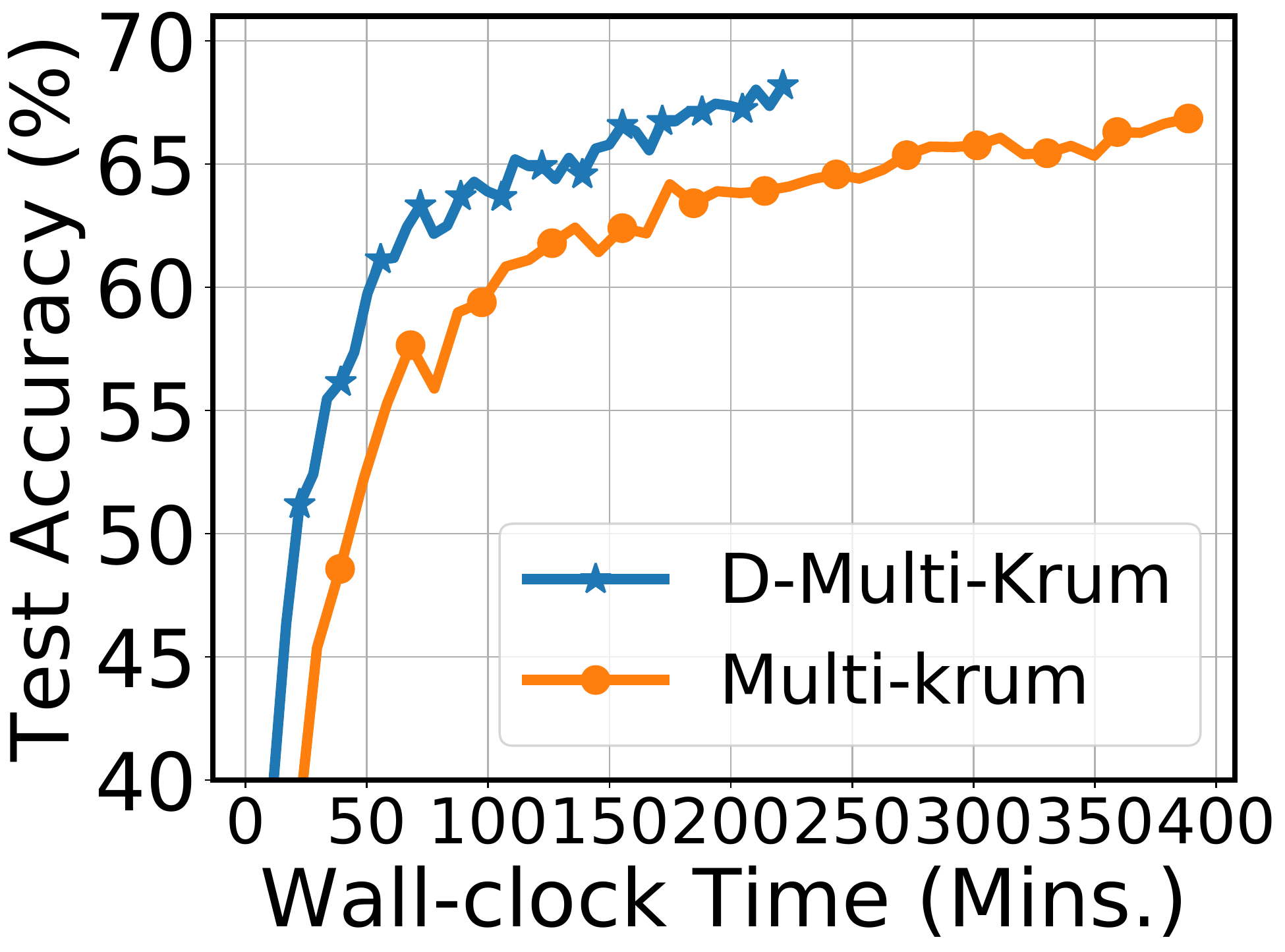}}
	\subfigure[VGG13-BN, \bulyan{}]{\includegraphics[width=0.31\textwidth]{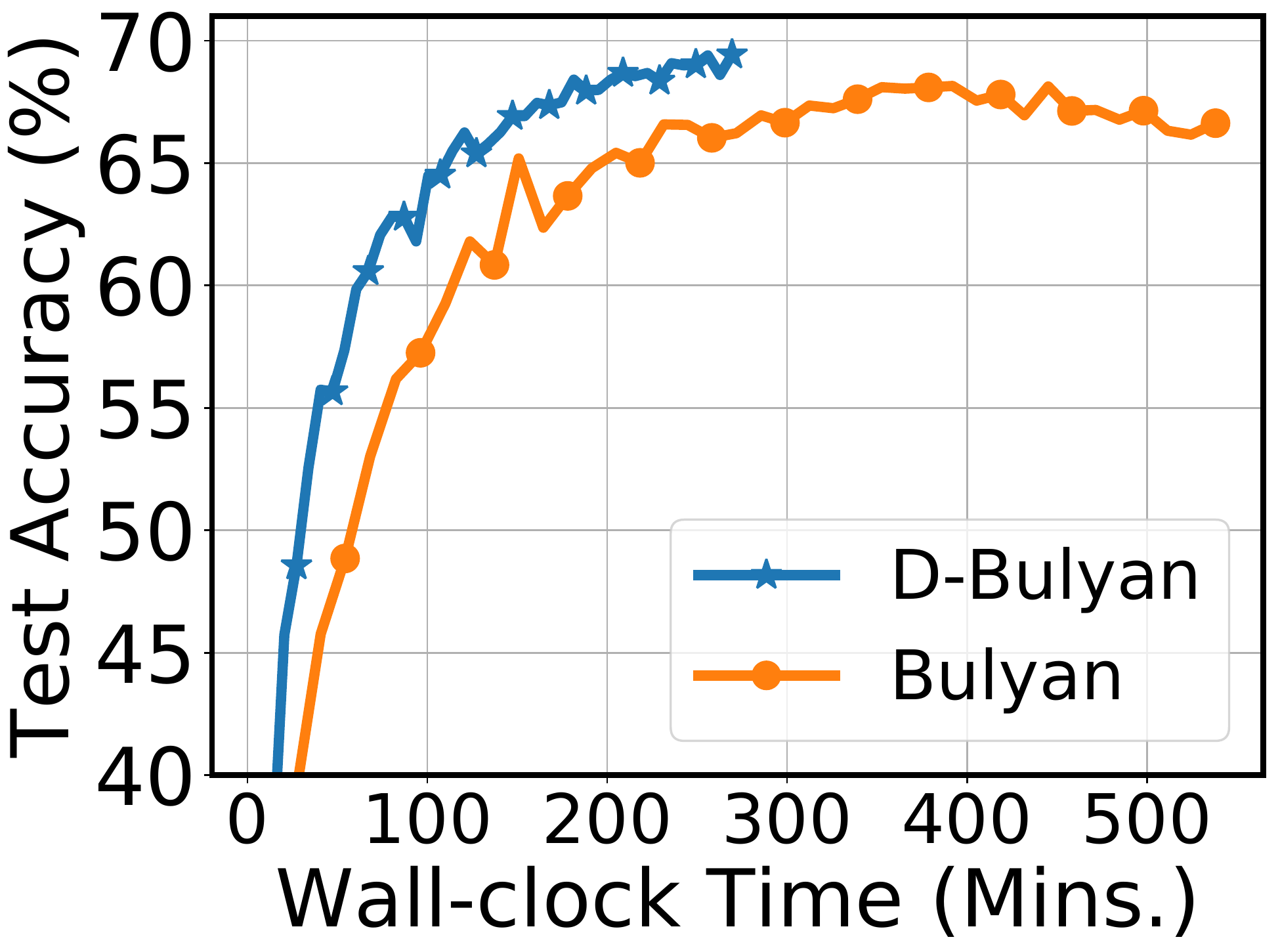}}
	\subfigure[VGG13-BN, \textit{Coord-Median}]{\includegraphics[width=0.31\textwidth]{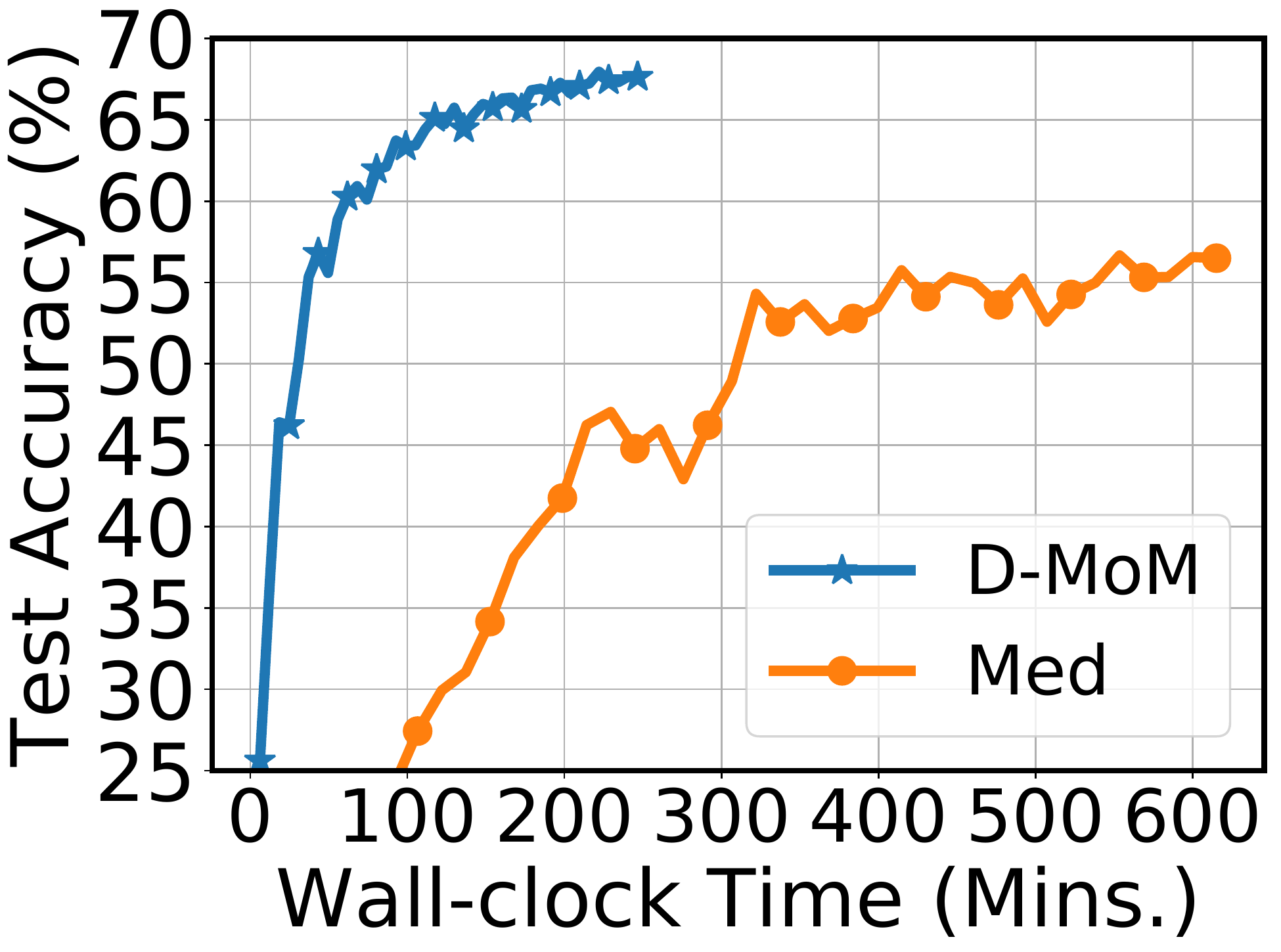}}
	\caption{End-to-end convergence comparisons among applying \dracolite{} on different baseline methods under \textit{reverse gradient} attack. (a)-(c): comparisons between vanilla and \dracolite{} deployed version of \multikrum{}, \bulyan{}, and coordinate-wise median over ResNet-18 trained on CIFAR-10. (d)-(f): same comparisons over VGG13-BN trained on CIFAR-100.}
	\label{fig:accRevGradRuntime}
\end{figure*}
\paragraph{Byzatine-resilience under various attacks}
\begin{table}[H]
	\caption{ Summary of defense results over ALIE attacks~\cite{baruch2019little}; the numbers reported correspond to test set prediction accuracy.}
	\label{table:byzantine_resilience}
	\begin{center}
		 \ra{1.3}
			\begin{tabular}{ccc}
				\toprule \textbf{Methods}
				& ResNet-18 &  VGG13-BN \bigstrut\\
				\midrule
				D-\multikrum{} & 80.3\% & 42.98\% \bigstrut\\
				D-\bulyan{} & 76.8\% & 46.82\%  \bigstrut\\
				D-Med. & \textbf{86.21\%} & \textbf{59.51\%}  \bigstrut\\
				\multikrum{} & 45.24\% & 17.18\%  \bigstrut\\
				\bulyan{} & 42.56\% & 11.06\% \bigstrut\\
				Med. & 43.7\% & 8.64\%  \bigstrut\\
				\bottomrule
			\end{tabular}%
	\end{center}
\end{table}

We first study the Byzantine-resilience of all methods and baselines under the ALIE attack, which is to the best of our knowledge, the strongest Byzantine attack known. The results on ResNet-18 and VGG13-BN are shown in Figure \ref{fig:introLieConvResNet} and  \ref{fig:breakdownAnalysisVGG} respectively. Applying \dracolite{} leads to significant improvement on Byzantine-resilience compared to vanilla \multikrum{}, \bulyan{}, and coordinate-wise median on both datasets as shown in Table \ref{table:byzantine_resilience}.

We then consider the \textit{reverse gradient} attack, the results are shown in Figure \ref{fig:accRevGradRuntime}. Since \textit{reverse gradient} is a much weaker attack, all vanilla robust aggregation methods and their \dracolite{} paired variants defend well.

Moreover, applying \dracolite{} leads to significant end-to-end speedups. In particular, combining the coordinate-wise median with \dracolite{} led to a $5\times$ speedup gain in the amount of time to achieve to 90\% test set prediction accuracy for ResNet-18 trained on CIFAR-10. The speedup results are shown in Figure \ref{fig:speedUp}. For the experiment where VGG13-BN was trained on CIFAR-100, up to an order of magnitude end-to-end speedup can be observed in coordinate-wise median applied on top of \dracolite{}. 
 
For completeness, we also compare versions of \dracolite{} with \textsc{Draco} \cite{chen2018draco}. This is not the focus of this work, as we are primarily interested in showing that \dracolite{} improves the robustness of traditional robust aggregators. However the comparisons with \textsc{Draco} can be found in the Appendix \ref{appendix:DracoComparisons}.

\begin{figure*}
	\centering      
	\subfigure[ResNet-18, CIFAR-10]{\includegraphics[width=0.29\textwidth]{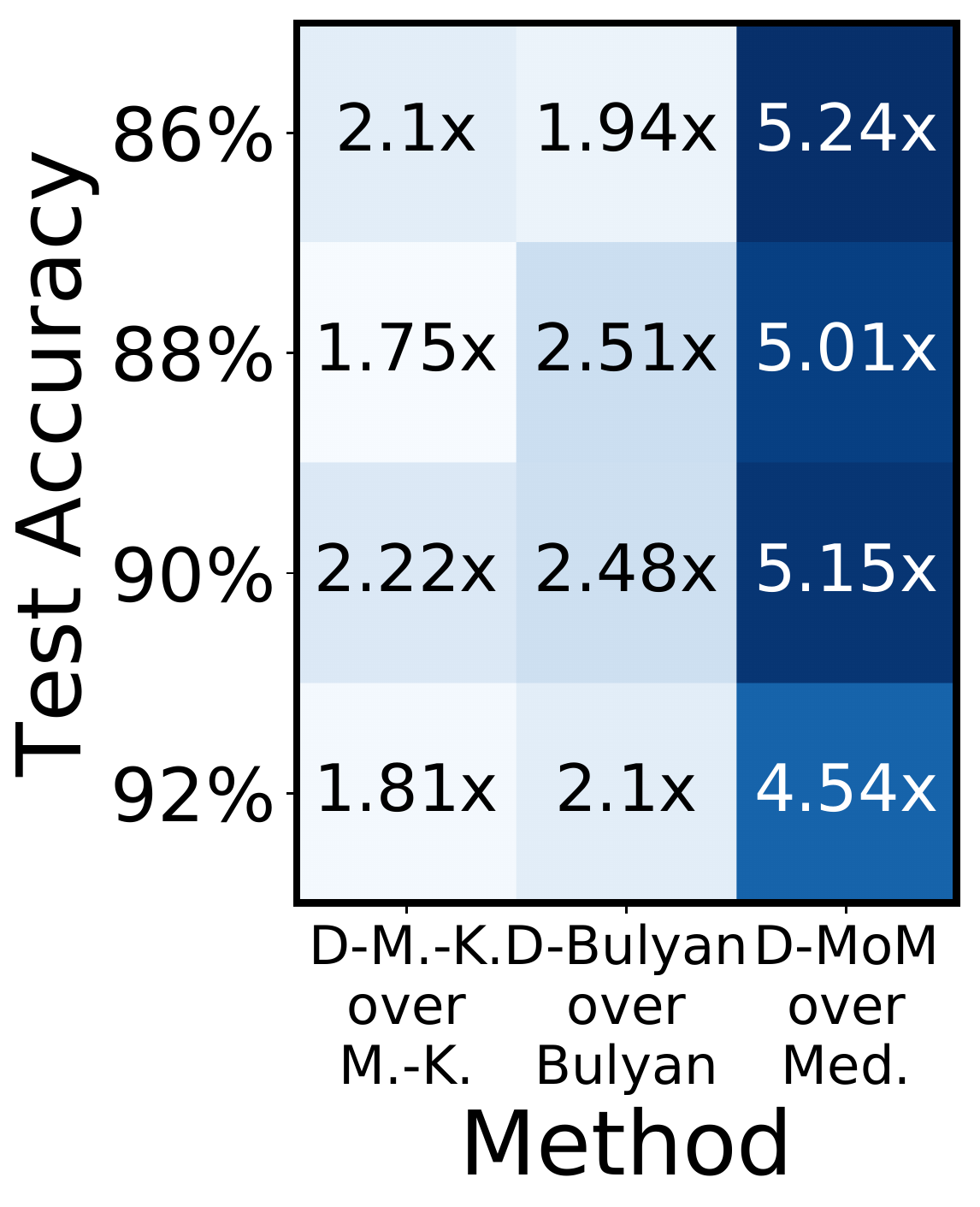}}
	\subfigure[VGG13-BN, CIFAR-100]{\includegraphics[width=0.38\textwidth]{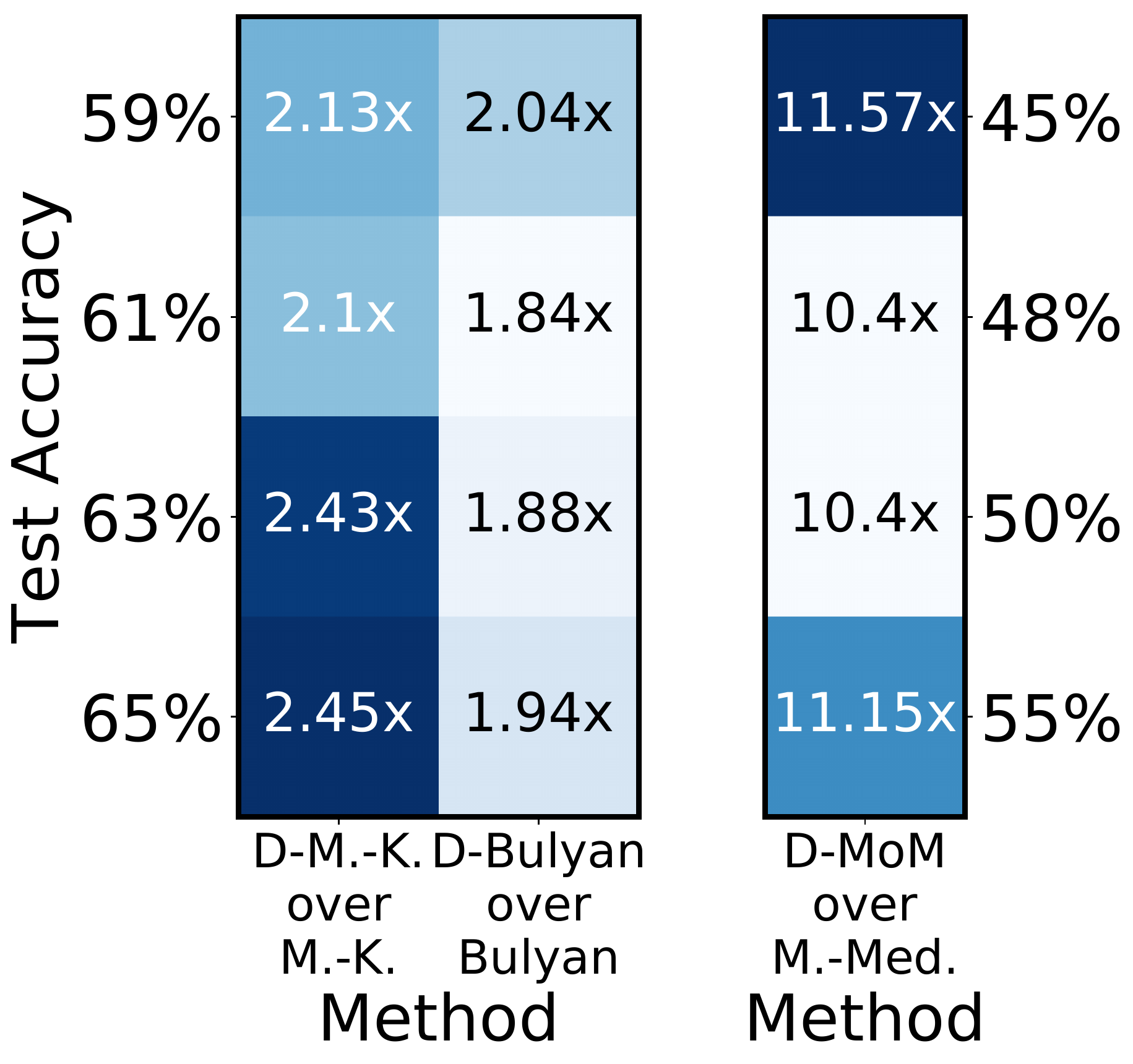}}
	\caption{Speedups in converging to specific accuracies for vanilla robust aggregation methods and their \dracolite{}-deployed variants under \textit{reverse gradient} attack: (a) results of ResNet-18 trained on CIFAR-10, (b) results of VGG13-BN trained on CIFAR-100}
	\label{fig:speedUp}
\end{figure*}


\paragraph{Comparison between \dracolite{} and \signum{}} We compare \dracolite{} paired \signum{} with vanilla \signum{} where only the sign information of each gradient element will be sent to the PS. The PS, on receiving sign information of gradients, takes coordiante-wise majority votes to get the model update. As is argued in \cite{bernstein2018signsgd}, the gradient distribution for many mordern deep networks can be close to unimodal and symmetric, hence a random sign flip attack is weak since it will not hurt the gradient distribution. We thus consider a stronger \textit{constant Byzantine attack} introduced in Section \ref{seubsec:implementation-of-detox}.
To pair \dracolite{} with \signum{}, after the majority voting stage of \dracolite{}, we set both $\mathcal{A}_0$ and $\mathcal{A}_1$ as coordinate-wise majority vote describe in Algorithm 1 in \cite{bernstein2018signsgd}.  For hyper-parameter tuning, we follow the suggestion in \cite{bernstein2018signsgd} and set the initial learning rate at $0.0001$. However, in defensing the our proposed \textit{constant Byzantine attack}, we observe that constant learning rates lead to model divergence. Thus, we tune the learning rate schedule and use $0.0001\times0.99^{t\pmod{10}}$ for both \dracolite{} and \dracolite{} paired \signum{}. 

The results of both ResNet-18 trained on CIFAR-10 and VGG13-BN trained on CIFAR-100 are shown in Figure \ref{fig:appendixComparesignSGD} where we observe that \dracolite{} paired \signum{} improves the Byzantine resilience of \signum{} significantly. For ResNet-18 trained on CIFAR-10, \dracolite{} improves testset prediction accuracy of vanilla \signum{} from $34.92\%$ to $78.75\%$. While for VGG13-BN trained on CIFAR-100, \dracolite{} improves testset prediction accuracy (TOP-1) of vanilla \signum{} from $2.12\%$ to $40.37\%$.


\begin{figure*}[htp]
	\centering      
	\subfigure[ResNet-18 on CIFAR-10]{\includegraphics[width=0.4\textwidth]{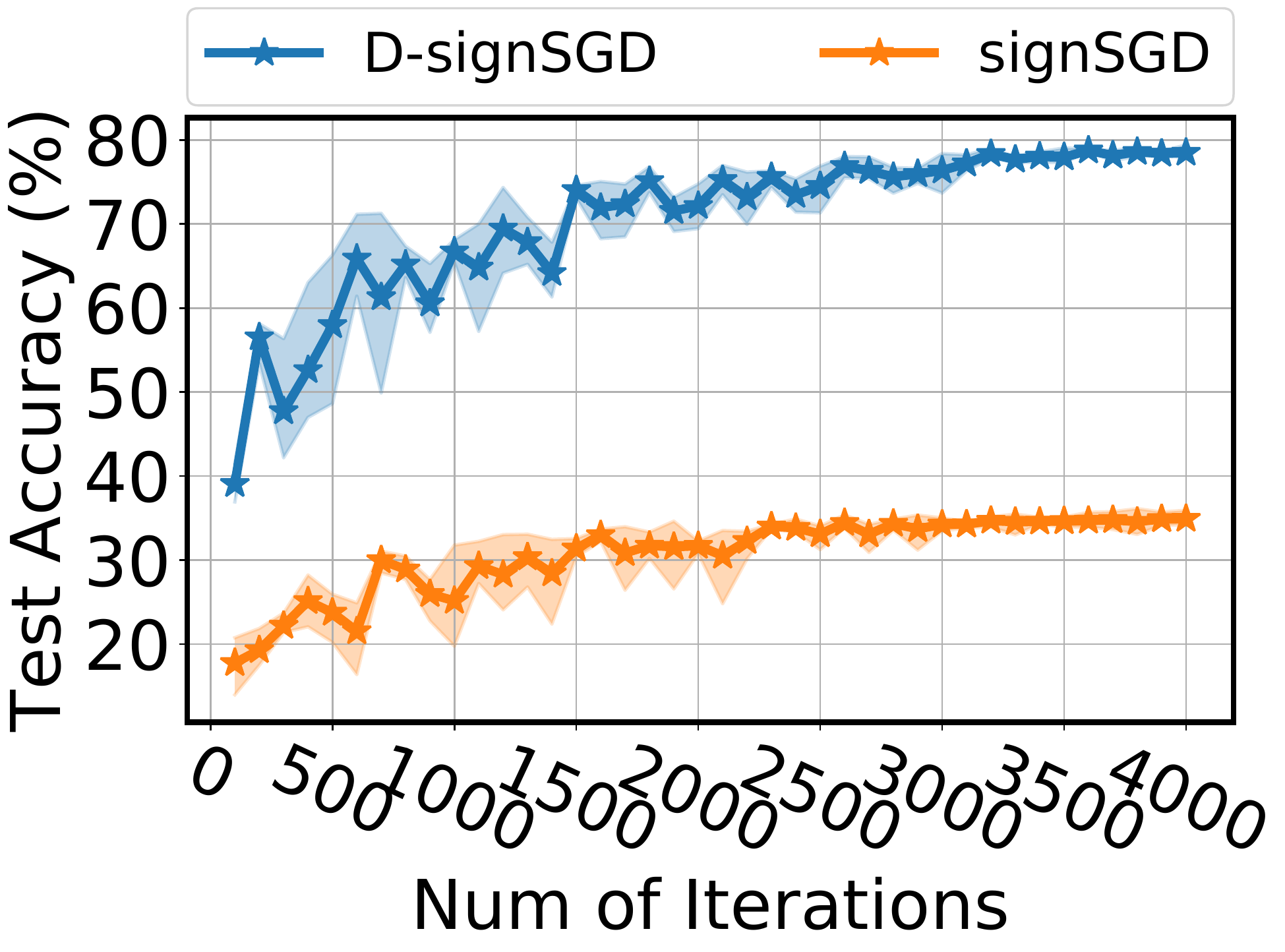}}
	\subfigure[VGG13-BN on CIFAR-100]{\includegraphics[width=0.4\textwidth]{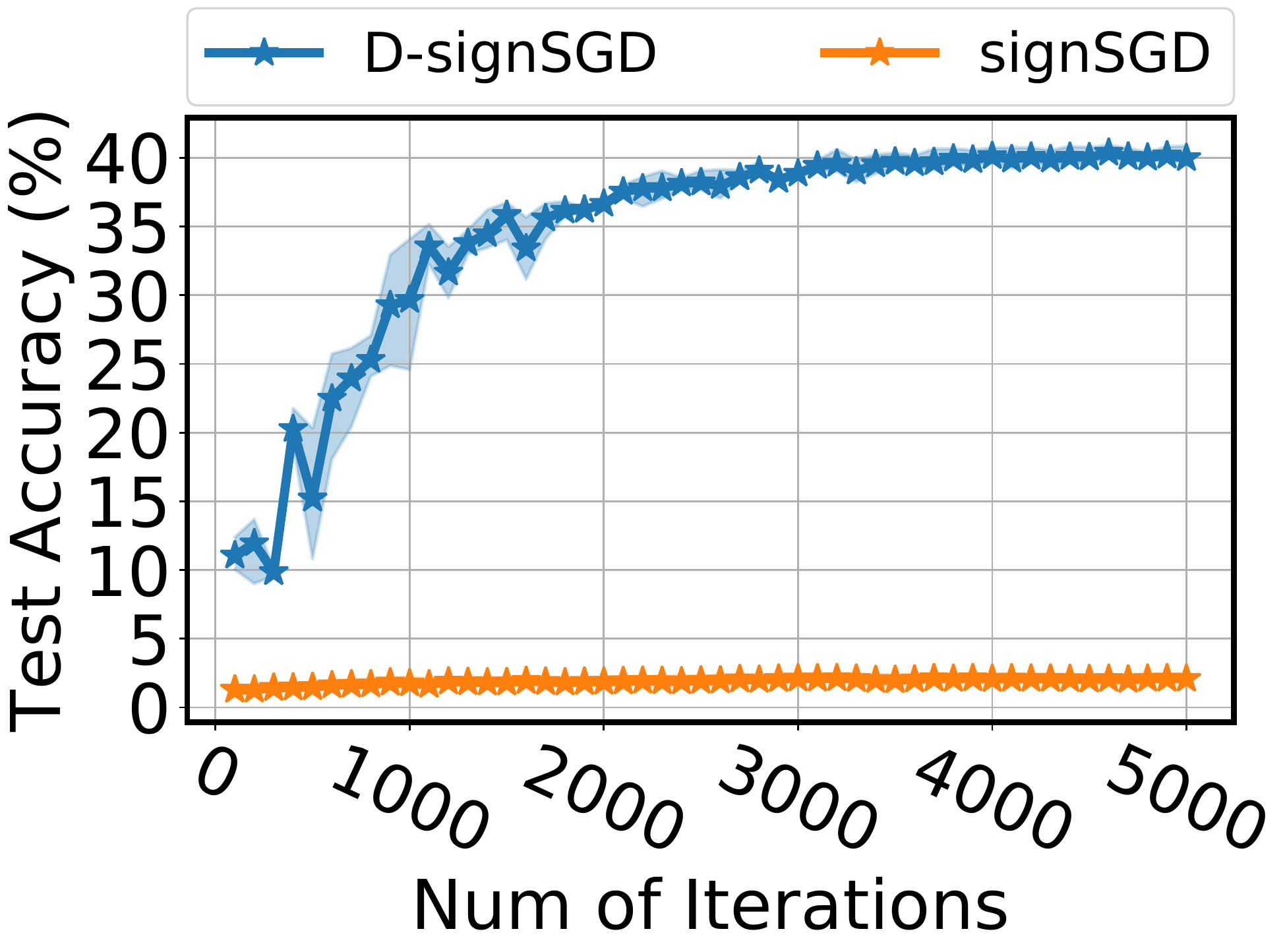}}
	\caption{Convergence comparisons among \dracolite{} paired with \signum{}  and vanilla \signum{} under \textit{constant Byzantine attack} on: (a) ResNet-18 trained on CIFAR-10 dataset; (b) VGG13-BN trained on CIFAR-100 dataset}
	\label{fig:appendixComparesignSGD}
\end{figure*}

	\begin{figure}[htp]
	\centering
	\includegraphics[width=0.45\textwidth]{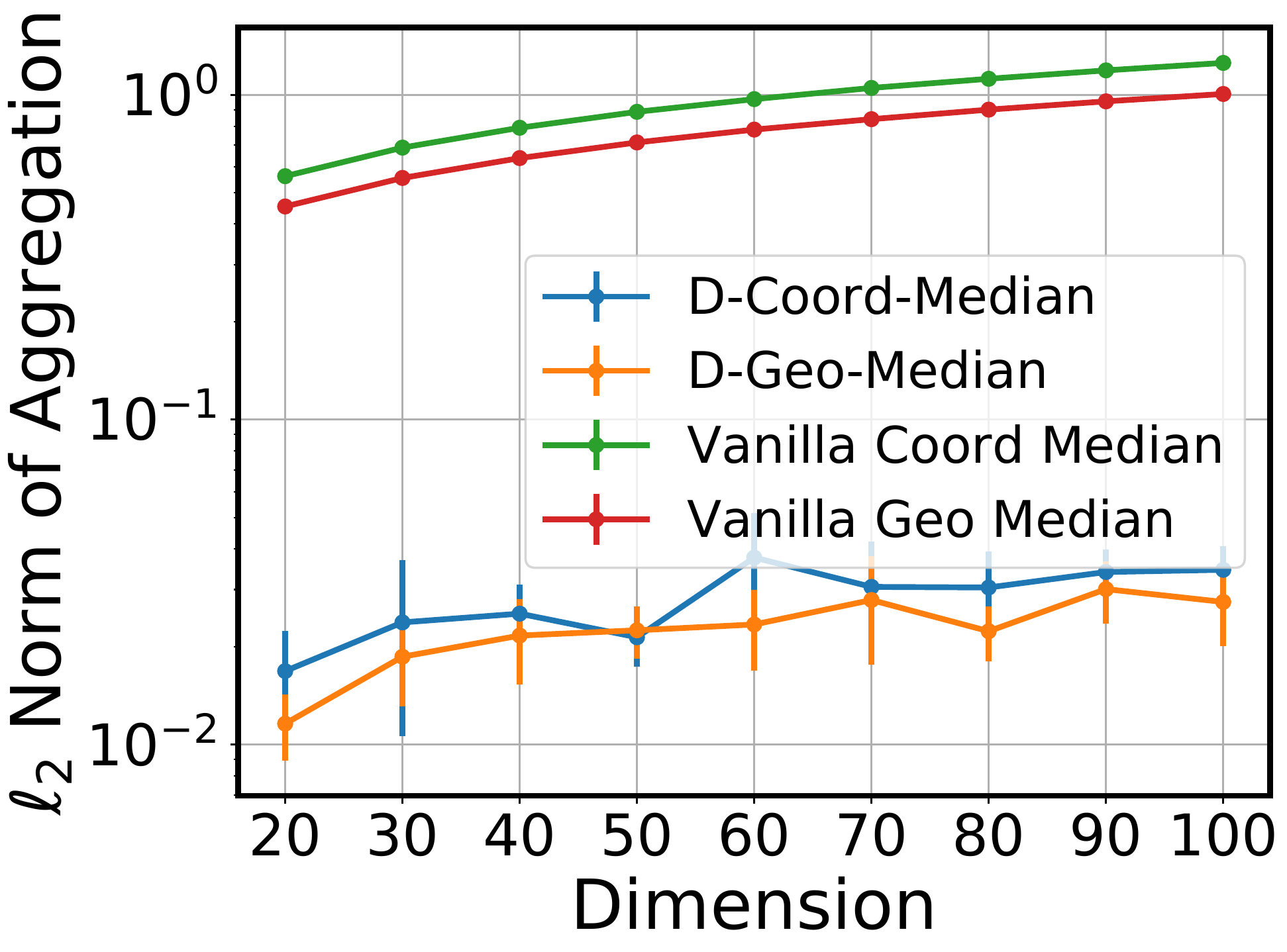}
	\caption{Experiment with synthetic data for robust mean estimation: error is reported against dimension (lower is better)}\label{fig:meanEst}

	\end{figure}

\paragraph{Mean estimation on synthetic data} 

To verify our theoretical analysis, we finally conduct an experiment for a simple mean estimation task. The result of our synthetic mean experiment are shown in Figure \ref{fig:meanEst}. In the synthetic mean experiment, we set $p=220000, r=11, q = \lfloor\frac{e^{r}}{3}\rfloor$, and for dimension $d \in \{20, 30, \cdots, 100\}$, we generate $20$ samples \textit{iid} from $\mathcal{N}(0,I_d)$. The Byzantine nodes, instead send a constant vector of the same dimension with $\ell_2$ norm of 100. The robustness of an estimator is reflected in the $\ell_2$ norm of its mean estimate. Our experimental results show that \dracolite{} increases the robustness of geometric median and coordinate-wise median, and decreases the dependecne of the error on $d$.

	\section{Conclusion}\label{sec:conclusions}
	\vspace{-0.3cm}
	In this paper, we present \dracolite{}, a new framework for Byzantine-resilient distributed training. Notably, any robust aggregator can be immediatley used with \dracolite{} to increase its robustness and efficiency. We demonstrate these improvements theoretically and empirically. In the future, we would like to devise a privacy-preserving version of \dracolite{}, as currently it requires the PS to be the owner of the data, and also to partition data among compute nodes. This means that the current version of \dracolite{} is not privacy preserving. Overcoming this limitation would allow us to develop variants of \dracolite{} for federated learning. 

	
	\bibliographystyle{unsrtnat} 
	\bibliography{draco_lite}
	
	\newpage
	\appendix
	
\section{Proofs}

\subsection{Proof of Theorem \ref{lemma:2}}
The following is a more precise statement of the theorem.
\begin{theorem*}
If $r>3$, $p\geq 2r$ and $\epsilon<1/40$ 
then $\E[\hat{q}]$ falls as $\mathcal{O}\left(q(40\epsilon(1-\epsilon))^{(r-1)/2}/r\right)$ which is exponential in r.
\end{theorem*}
\begin{proof}By direct computation,
\begin{align*}
	\E(\hat{q}) &= \E\left(\sum_{i=1}^{p/r} X_i\right)\\
    &= \dfrac{p}{r}\E(X_i)&\\
    &= \dfrac{p}{r} \dfrac{\displaystyle\sum_{i=0}^{(r-1)/2}\binom{q}{r-i}\binom{p-q}{i}}{\displaystyle\binom{p}{r}}&\\
    &\leq \dfrac{p}{r}\dfrac{\displaystyle\frac{r+1}{2}\binom{q}{(r+1)/2}\binom{p-q}{(r-1)/2}}{\displaystyle\binom{p}{r}}&\\
    &\leq \dfrac{p}{r}\dfrac{r+1}{2}\dfrac{\displaystyle \binom{r}{(r-1)/2}q^{(r+1)/2}(p-q)^{(r-1)/2}}{(p-r)^r}&\\
    &= \dfrac{p}{r}\frac{r+1}{2}\dfrac{\displaystyle\binom{r}{(r-1)/2}q^{(r+1)/2}(p-q)^{(r-1)/2}}{p^r(1-r/p)^r}&\\
    &\leq \dfrac{p}{r}\frac{r+1}{2}\dfrac{\displaystyle\binom{r}{(r-1)/2}q^{(r+1)/2}(p-q)^{(r-1)/2}}{p^r(1/2)^r}&\\
    &=\dfrac{p}{r}(r+1)2^{r-1}\binom{r}{(r-1)/2}\epsilon^{(r+1)/2}(1-\epsilon)^{(r-1)/2}.&
\end{align*}

Note that $\binom{r}{(r-1)/2}$ is the coefficient of $x^{(r+1)/2}(1-x)^{(r-1)/2}$ in the binomial expansion of $1 = 1^r = (x + (1-x))^r$. Therefore, setting $x=\frac{1}{2}$, we find that $\binom{r}{(r-1)/2} \leq 2^r$. Therefore,

\begin{align*}    
    & \dfrac{p}{r}(r+1)2^{r-1}\binom{r}{(r-1)/2}\epsilon^{(r+1)/2}(1-\epsilon)^{(r-1)/2}&\\
    &\leq \dfrac{p}{r}(r+1)2^{2r-1}\epsilon^{(r+1)/2}(1-\epsilon)^{(r-1)/2}&\\
    &= \dfrac{p}{r}(r+1)\epsilon \bigg(2^{2r-1} \epsilon^{(r-1)/2}(1-\epsilon\bigg)^{(r-1)/2})&\\
    &= \dfrac{2q}{r}(r+1)\bigg(16\epsilon(1-\epsilon)\bigg)^{(r-1)/2}&\\
    &= \dfrac{2q}{r}\bigg(16(r+1)^{2/(r-1)}\epsilon(1-\epsilon)\bigg)^{(r-1)/2}&.\end{align*}

Note that since $r > 3$ and $r$ is odd, we have $r \geq 5$. Therefore,
    \begin{align*} 
    \E(\hat{q}) \leq 2q(40\epsilon(1-\epsilon))^{(r-1)/2}/r.
    \end{align*}
\end{proof}

For $r=3$, we have the following lemma.
\begin{lemma}
 If $r=3$, then $\E[\hat{q}]\leq q(4\epsilon-2\epsilon^2)/3$  when $n\geq 6$.
\end{lemma}
\begin{proof}    
\begin{align*}
\E(q_e) &= \E(\sum_{i=1}^{\frac{p}{3}} X_i) = \dfrac{p}{3}E(X_i) = \dfrac{p}{3} \dfrac{ \binom{q}{3} + \binom{q}{2}\binom{p-q}{1}}{\binom{n}{3}}&\\
    &= \dfrac{p}{3}\dfrac{q(q-1)(3p-2q-2)}{p(p-1)(p-2)}=\dfrac{q}{3}\dfrac{\left(\epsilon-\frac{1}{p}\right)\left(3-2\epsilon-\frac{2}{p}\right)}{\left(1-\frac{1}{p}\right)\left(1-\frac{2}{p}\right)}&\\
    &\leq \dfrac{q}{3}\epsilon\dfrac{3-2\epsilon-\frac{2}{p}}{1-\frac{2}{p}}\leq q\epsilon(4-2\epsilon)/3
\end{align*}
\end{proof}

\subsection{Proof of Corollary \ref{lemma:5}}
From Theorem \ref{lemma:2} we see that $\E[\hat{q}] \leq 2q(40\epsilon(1-\epsilon))^{(r-1)/2}/r\leq 2q(40\epsilon)^{(r-1)/2}$. Now, straightforward analysis implies that if $\epsilon\leq 1/80$ and $r \geq 3 + 2\log_2 q$ 
 then $\E[\hat{q}]\leq 1$.  
We will then use the following Lemma:
\begin{lemma}\label{lemma:3}For all $\theta > 0$,
    $$\mathbb{P}\left[\hat{q} \geq \E[\hat{q}](1 + \theta)\right] \leq  \left(\dfrac{1}{1+\theta/2}\right)^{ \E[\hat{q}] \theta/2}$$
\end{lemma}

Now, using Lemma \ref{lemma:3} and assuming $\theta\geq 2$,
\begin{align*}
\mathbb{P}\left[\hat{q} \geq \E[\hat{q}](1 + \theta)\right] \leq \left(\dfrac{1}{1+\theta/2}\right)^{ \E[\hat{q}] \theta/2}\\
\implies \mathbb{P}\left[\hat{q} \geq 1 + \E[\hat{q}] \theta\right] \leq \left(\dfrac{1}{1+\theta/2}\right)^{ \E[\hat{q}] \theta/2}\\
\implies \mathbb{P}\left[\hat{q} \geq 1 + \E[\hat{q}] \theta\right] \leq 2^{ -\E[\hat{q}] \theta/2}\\
\end{align*}
where we used the fact that $\E[\hat{q}] \leq 1$ in the first implication and the assumption that $\theta \geq 2$ in the second. Setting $\delta:=2^{ -\E[\hat{q}] \theta/2}$, we get the probability bound. Finally, setting $\delta \leq 1/2$ makes $\theta\geq 2$, which completes the proof.

\subsection{Proof of Lemma \ref{lemma:3}}


We will prove the following: 
$$P\left[\hat{q} \geq \E[\hat{q}](1 + \theta)\right] \leq  \left(\dfrac{1}{1+\dfrac{\theta}{2}}\right)^{ \E[\hat{q}] \theta/2}$$

\begin{proof}

We will use the following theorem for this proof \cite{linial14,pelekis17}.

\begin{theorem*}[Linial~\cite{linial14}]\label{Linial}
Let $X_1,\ldots,X_{\hat{p}}$ be Bernoulli $0/1$ random variables. Let $\beta\in (0,1)$ be such that $\beta \hat{p}$ is a positive integer and let $k$ be 
any positive integer such that $0< k < \beta \hat{p}$. Then 
$$ \mathbb{P}\left[\sum_{i=1}^{\hat{p}}X_i\geq \beta \hat{p}\right] \leq \frac{1}{\binom{\beta \hat{p}}{k}} \sum_{|A|=k} \mathbb{P}\left[\land_{i\in A}(X_i=1)\right] $$
\end{theorem*}

Let $\beta \hat{p} = \mathbb{E}[\hat{q}](1+\theta)$. Now, $\mathbb{P}[X_i=1] = \E[X_i] = \E[\hat{q}]/\hat{p}$.
We will show that
$$\mathbb{P}\left[\land_{i\in A}(X_i=1)\right]  \leq (\E[\hat{q}]/\hat{p})^k$$ 
where $A \subseteq \{1, \dots, \hat{p}\} $ of size $k$. To see this, note that for any $i$, $\mathbb{P}[X_i=1]  = \E[\hat{q}]/\hat{p}$. The conditional probability of some other $X_j$ being $1$ given that $X_i$ is $1$ would only reduce. Formally, for $i \neq j$,
$$\mathbb{P}[X_j=1|X_i=1] \leq \mathbb{P}[X_i=1] = \epsilon\gamma.$$
Note that for $X_i$ to be $1$, the Byzantine machines in the $i$-th block must be in the majority. Hence, the reduction in the pool of leftover Byzantine machines was more than honest machines. Since the total number of Byzantine machines is less than the number of honest machines, the probability for them being in a majority in block $j$ reduces. Therefore,

\begin{align*}
\mathbb{P}\left[\sum_{i=1}^{\hat{p}}X_i \geq \mathbb{E}[\hat{q}](1 + \theta)\right]&\leq \dfrac{\displaystyle\binom{\hat{p}}{k}}{\displaystyle\binom{\mathbb{E}[\hat{q}](1 + \theta)}{k}} \mathbb{P}\left[\land_{i\in A}(X_i=1)\right]&\\
&\leq \dfrac{\displaystyle\binom{\hat{p}}{k}}{\displaystyle\binom{\mathbb{E}[\hat{q}](1 + \theta)}{k}} (\E[\hat{q}]/\hat{p})^k &\\
&\leq \dfrac{\displaystyle(\hat{p})^k}{\displaystyle k!\binom{\mathbb{E}[\hat{q}](1 + \theta)}{k}} \left(\dfrac{\E[\hat{q}]}{\hat{p}}\right)^k &.\end{align*}

Letting $k = \mathbb{E}[\hat{q}]\theta/2$, we then have
\begin{align*}
\mathbb{P}\left[\sum_{i=1}^{\hat{p}}X_i \geq \mathbb{E}[\hat{q}](1 + \theta)\right]&\leq \frac{(\hat{p})^k}{(\mathbb{E}[\hat{q}](1 + \theta/2))^k} (\E[\hat{q}]/\hat{p})^k &\\
&= \left(\frac{1}{1+\frac{\theta}{2}}\right)^{\mathbb{E}[\hat{q}]\theta/2} &
\end{align*}

\end{proof}

\subsection{Proof of Theorem \ref{prop:MoM}}
We will adapt the techniques of Theorem 3.1 in \cite{minsker2015geometric}.
\begin{lemma}[\cite{minsker2015geometric}, Lemma 2]
Let $\mathbb{H}$ be some Hilbert space, and for $x_1,\dots, x_k \in \mathbb{H}$, let $x_{gm}$ be their geometric median. Fix $\alpha \in (0, \frac{1}{2})$ and suppose that $z \in \mathbb{H}$ satisfies $\|x_{gm} - z\| > C_\alpha r$, where
\begin{equation*}
C_\alpha = (1-\alpha)\sqrt{\dfrac{1}{1-2\alpha}}
\end{equation*}
and $r>0$. Then there exists $J \subseteq \{1, \dots, k\}$ with $|J|>\alpha k$ such that for all $j\in J$, $\|x_j-z\|>r$.
\end{lemma}
Note that for a general Hilbert or Banach space $\mathbb{H}$, the geometric median is defined as:
$$x_{gm} := \arg\min \sum_{j=1}^k \|x-x_j\|_{\mathbb{H}}$$
where $\|.\|_{\mathbb{H}}$ is the norm on $\mathbb{H}$. This coincides with the notion of geometric median in $\mathbb{R}^2$ under the $\ell_2$ norm. Note that Coordinatewise Median is the Geometric Median in the real space with the $\ell^1$ norm, which forms a Banach space.

Firstly, we use Corollary \ref{lemma:5} to see that with probability $1-\delta$, $\hat{q}\leq1 + 2 \log(1/\delta)$. Now, we assume that $\hat{q}\leq1 + 2 \log(1/\delta)$ is true. Conditioned on this event, we will show the remainder of the theorem holds with probability at least $1-\delta$. Hence, with total probability at least $(1-\delta)^2\geq 1-2\delta$, the statement of the theorem holds.

\textbf{(1):}~Let us assume that number of clusters is $k = 128 \log 1/\delta$ for some $\delta < 1$, also note that because $\delta \in[0,1/2]$, we have that $k=128 \log 1/\delta \geq 64 (0.5 + \log 1/\delta)  \geq 8\hat{q}$. Now, choose $\alpha = 1/4$. Choose $r=4\sigma \sqrt{\frac{k}{b}}$. Assume that the Geometric Median is more than $C_\alpha r$ distance away from true mean. Then by the previous Lemma, atleast $\alpha = 1/4$ fraction of the empirical means of the clusters must lie atleast $r$ distance away from true mean. Because we assume the number of clusters is more than $8\hat{q}$, atleast $1/8$ fraction of empirical means of uncorrupted clusters must also lie atleast $r$ distance away from true mean.

Recall that the variance of the mean of an ``honest'' vote group is given by
$$(\sigma')^2=\sigma^2\dfrac{k}{b}.$$
By applying Chebyshev's inequality to the $i^{th}$ uncorrupted vote group $G[i]$, we find that its empirical mean $\hat{x}$ satisfies 
$$\mathbb{P}\left(\|G[i]-G\| \geq 4\sigma \sqrt{\dfrac{k}{b}}\right)\leq \dfrac{1}{16}.$$

Now, we define a Bernoulli event that is 1 if the empirical mean of an uncorrupted vote group is at distance larger than $r$ to the true mean, and 0 otherwise. By the computation above, the probability of this event is less than $1/16$. Thus, its mean is less than $1/16$ and we want to upper bound the probability that empirical mean is more than $1/8$. Using the number of events as $k = 128 \log(1/\delta)$, we find that this holds with probability at least $1-\delta$. For this, we used the following version of Hoeffding's inequality in this part and part (3) of this proof. For Bernoulli events with mean $\mu$, empirical mean $\hat{\mu}$, number of events $m$ and deviation $\theta$:
$$\mathbb{P}(\hat{\mu}-\mu\geq\theta)\leq \exp(-2m\theta^2)$$
 To finish the proof, just plug in the values of $C_\alpha$ given in the Lemma 2.1 (written above) from \cite{minsker2015geometric}, where $C_\alpha = 3/2\sqrt{2}$ for Geometric Median.

\textbf{(2):}~For coordinate-wise median, we set $k = 128 \log d/\delta$. Then we apply the result proved in previous part for each dimension of $\hat{G}$. Then, we get that with probability at least $1-\delta/d$, 
$$|\hat{G}_i-G_i| \leq C_1  \sigma_i \sqrt{\frac{\log d/\delta}{b}}$$
where $\hat{G}_i$ is the $i^{th}$ coordinate of $\hat{G}$, ${G}_i$ is the $i^{th}$ coordinate of ${G}$ and $\sigma_i^2$ is the $i^{th}$ diagonal entry of $\Sigma$. Doing a union bound, we get that with probability at least $1-\delta/d$
$$\|\hat{G}-G\| \leq C_1 \sigma \sqrt{\frac{\log d/\delta}{b}}.$$

\textbf{(3):}~Define
$$\Delta_i = \sigma_i \sqrt{\dfrac{k}{b\sqrt{\dfrac{1}{2k}\log \dfrac{d}{\delta}}}}$$
where $\sigma_i^2$ is the $i^{th}$ diagonal entry of $\Sigma$. Now, for each uncorrupted vote group, using Chebyshev's inequality:
$$\mathbb{P}\left (|\hat{G}_i-G_i|\geq \Delta_i\right)\leq \sqrt{\dfrac{1}{2k}\log \dfrac{d}{\delta}}.$$

Now, $i^{th}$ coordinate of $\alpha$-trimmed mean lies $\Delta_i$ away from $G_i$ if atleast $\alpha k$ of the $i^{th}$ coordinates of vote group empirical means lie $\Delta_i$ away from $G_i$. Note that because of the assumption of the Proposition $\alpha k \geq 2 \hat{q}$. Because $\hat{q}$ of these can be corrupted, atleast $\alpha k/2$ of true empirical means have $i^{th}$ coordinates that lie $\Delta_i$ away from $G_i$. This means $\alpha/2$ fraction have true empirical means have $i^{th}$ coordinates that lie $\Delta_i$ away from $G_i$. Define a Bernoulli variable $X$ for a vote group as being 1 if the $i^{th}$ coordinate of empirical mean of that vote group lies more than $\Delta_i$ away from $G_i$, and 0 otherwise.

The mean of $X$ therefore satisfies
$$\E(X) < \sqrt{\dfrac{1}{2k}\log \dfrac{d}{\delta}}.$$
Set
$$\alpha = 4\sqrt{\dfrac{1}{2k}\log \dfrac{d}{\delta}}.$$
Again, using Hoeffding's inequality in a manner analogous to part (1) of the proof, we get that probability of $i^{th}$ coordinate of $\alpha$-trimmed mean being more than $\Delta_i$ away from $G_i$ is less than $\delta/d$.

Taking union bound over all $d$ coordinates, we find that the probability of $\alpha$-trimmed mean being more than
$$\sigma \sqrt{\dfrac{k}{b\sqrt{\dfrac{1}{2k}\log \dfrac{d}{\delta}}}}=\sigma \sqrt{\dfrac{4k}{b\alpha}}$$
away from $G$ is less than $\delta$. Hence we have proved that if
$$\alpha = 4\sqrt{\dfrac{1}{2k}\log \dfrac{d}{\delta}}$$
and $\alpha k \geq 2 \hat{q}$, then with probability at least $1-\delta$, $\Delta \leq \sigma \sqrt{\dfrac{4k}{b\alpha}} $. Now, set $\alpha=1/4$ and $k=128 \log (d/\delta)$. One can easily see that $\alpha k \geq 2 \hat{q}$ is satisfied and we get that with probability at least $1-\delta$, for some constant $C_3$,
$$\Delta \leq C_3 \sigma \sqrt{\dfrac{\log (d/\delta)}{b}}.$$

	\section{Extra Experimental Details}\label{appendix:extraExps}
\subsection{Implementation and system-level optimization details}\label{appendix:implem}
We introduce the details of combining \bulyan{}, \multikrum{}, and coordinate-wise median with \dracolite{}.
\begin{itemize}
	\item \bulyan{}: according to \cite{mhamdi2018hidden} \bulyan{} requires $p\geq 4q+3$. In \dracolite{}, after the first majority voting level, the corresponding requirement in \bulyan{} becomes $\frac{p}{r} \geq 4 \hat q +3 = 11$. Thus, we assign all ``winning" gradients in to one cluster \ie, \bulyan{} is conducted across 15 gradients.
	\item \multikrum{}: according to \cite{blanchard17}, \multikrum{} requires $p\geq 2q+3$. Therefore, for similar reason, we assign 15 ``winning" gradients into two groups with uneven sizes at 7 and 8 respectively.
	\item coordinate-wise median: for this baseline we follow the theoretical analysis in Section \ref{section:reduceByzantine} \ie, 15 ``winning" gradients are evenly assigned to 5 clusters with size at 3 for \textit{reverse gradient} Byzantine attack. For ALIE attack, we assign those 15 gradients evenly to 3 clusters with size of 5. The reason for this choice is simply that we observe the reported strategies perform better in our experiments. Then mean of the gradients is calculated in each cluster. Finally, we take coordinate-wise median across means of all clusters.
\end{itemize}
One important thing to point out is that we conducted system level optimizations on implementing \multikrum{} and \bulyan{}, \eg, parallelizing the computationally heavy parts in order to make the comparisons more fair according to \cite{aggregathor2019}. The main idea of our system-level optimization are two-fold: i) gradients of all layers of a neural network are firstly vectorized and concatenated to a high dimensional vector. Robust aggregations are then deployed on those high dimensional gradient vectors from all compute nodes. ii) As computational heavy parts exist for several methods \eg, calculating medians in the second stage of \bulyan{}. To optimize that part, we chunk the high dimensional gradient vectors evenly into pieces, and parallelize the median calculations in all the pieces. Our system-level optimization leads to 2-4 $\times$ speedup in the robust aggregation stage.

\subsection{Hyper-parameter tuning}\label{appendix:paramTun}
\begin{table}[ht]
	\centering
	\ra{1.3}
	\caption{Tuned stepsize schedules for experiments under \textit{reverse gradient} Byzantine attack}
	\begin{tabular}{ccc}
		\toprule Experiments
		& CIFAR-10 on ResNet-18 & CIFAR-100 on VGG13-BN \bigstrut\\
		\midrule
		D-\multikrum{} & 0.1 & 0.1 \bigstrut\\
		D-\bulyan{} & 0.1 & 0.1  \bigstrut\\
		D-Med. & $0.1\times 0.99^{t\pmod{10}}$ & $0.1\times 0.99^{t\pmod{10}}$  \bigstrut\\
		\multikrum{} & 0.03125 & 0.03125  \bigstrut\\
		\bulyan{} & 0.1 & 0.1 \bigstrut\\
		Med. & 0.1 & $0.1\times 0.995^{t\pmod{10}}$  \bigstrut\\
		\bottomrule
	\end{tabular}%
	\label{tab:stepsizeRevGrad}%
\end{table}%

\begin{table}[ht]
	\centering
	\ra{1.3}
	\caption{Tuned stepsize schedules for experiments under ALIE Byzantine attack}
	\begin{tabular}{ccc}
		\toprule Experiments
		& CIFAR-10 on ResNet-18 & CIFAR-100 on VGG13-BN \bigstrut\\
		\midrule
		D-\multikrum{} & $0.1\times 0.98^{t\pmod{10}}$ & $0.1\times 0.965^{t\pmod{10}}$ \bigstrut\\
		D-\bulyan{} & $0.1\times 0.99^{t\pmod{10}}$ & $0.1\times 0.965^{t\pmod{10}}$  \bigstrut\\
		D-Med. & $0.1\times 0.98^{t\pmod{10}}$ & $0.1\times 0.98^{t\pmod{10}}$  \bigstrut\\
		\multikrum{} & $0.0078125\times 0.96^{t\pmod{10}}$ & $0.00390625\times 0.965^{t\pmod{10}}$  \bigstrut\\
		\bulyan{} & $0.001953125\times 0.95^{t\pmod{10}}$ & $0.00390625\times 0.965^{t\pmod{10}}$ \bigstrut\\
		Med. & $0.001953125\times0.95^{t\pmod{10}}$ & $0.001953125\times0.965^{t\pmod{10}}$  \bigstrut\\
		\bottomrule
	\end{tabular}%
	\label{tab:stepsizeALIE}%
\end{table}%

\subsection{Data augmentation and normalization details}\label{appendix:dataPreprocess}
In preprocessing the images in CIFAR-10/100 datasets, we follow the standard data augmentation and normalization process. For data augmentation, random cropping and horizontal random flipping are used. Each color channels are normalized with mean and standard deviation by $\mu_r = 0.491372549, \mu_g = 0.482352941, \mu_b =  0.446666667$, $\sigma_r = 0.247058824, \sigma_g = 0.243529412, \sigma_b = 0.261568627$. Each channel pixel is normalized by subtracting the mean value in this color channel and then divided by the standard deviation of this color channel.


\subsection{Comparison between \dracolite{} and \textsc{Draco}} \label{appendix:DracoComparisons}
We provide the experimental results in comparing  \dracolite{} with \textsc{Draco}.

\begin{figure*}[htp]
	\centering      
	\subfigure[ResNet-18 on CIFAR-10]{\includegraphics[width=0.4\textwidth]{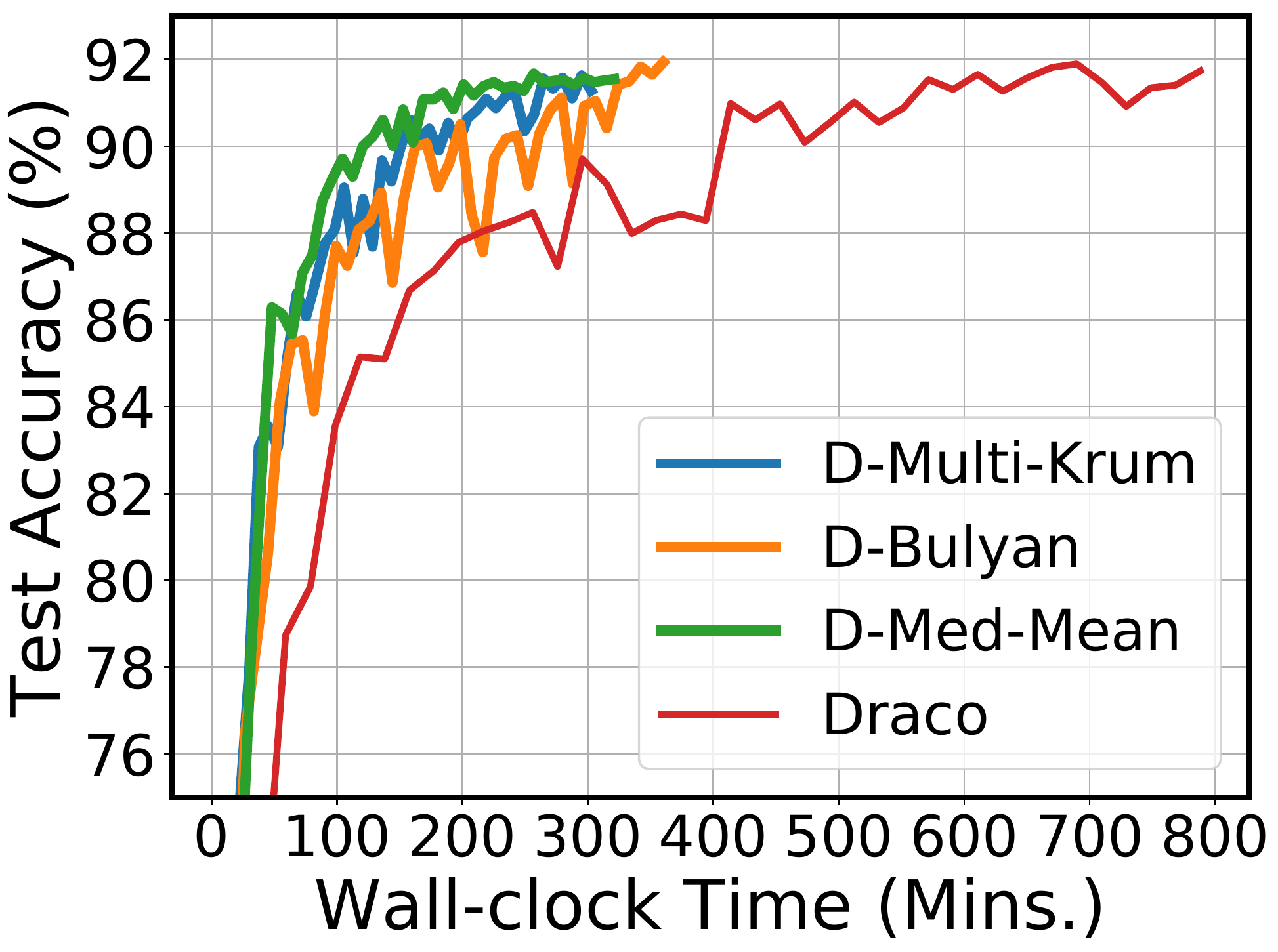}}
	\subfigure[VGG13-BN on CIFAR-100]{\includegraphics[width=0.4\textwidth]{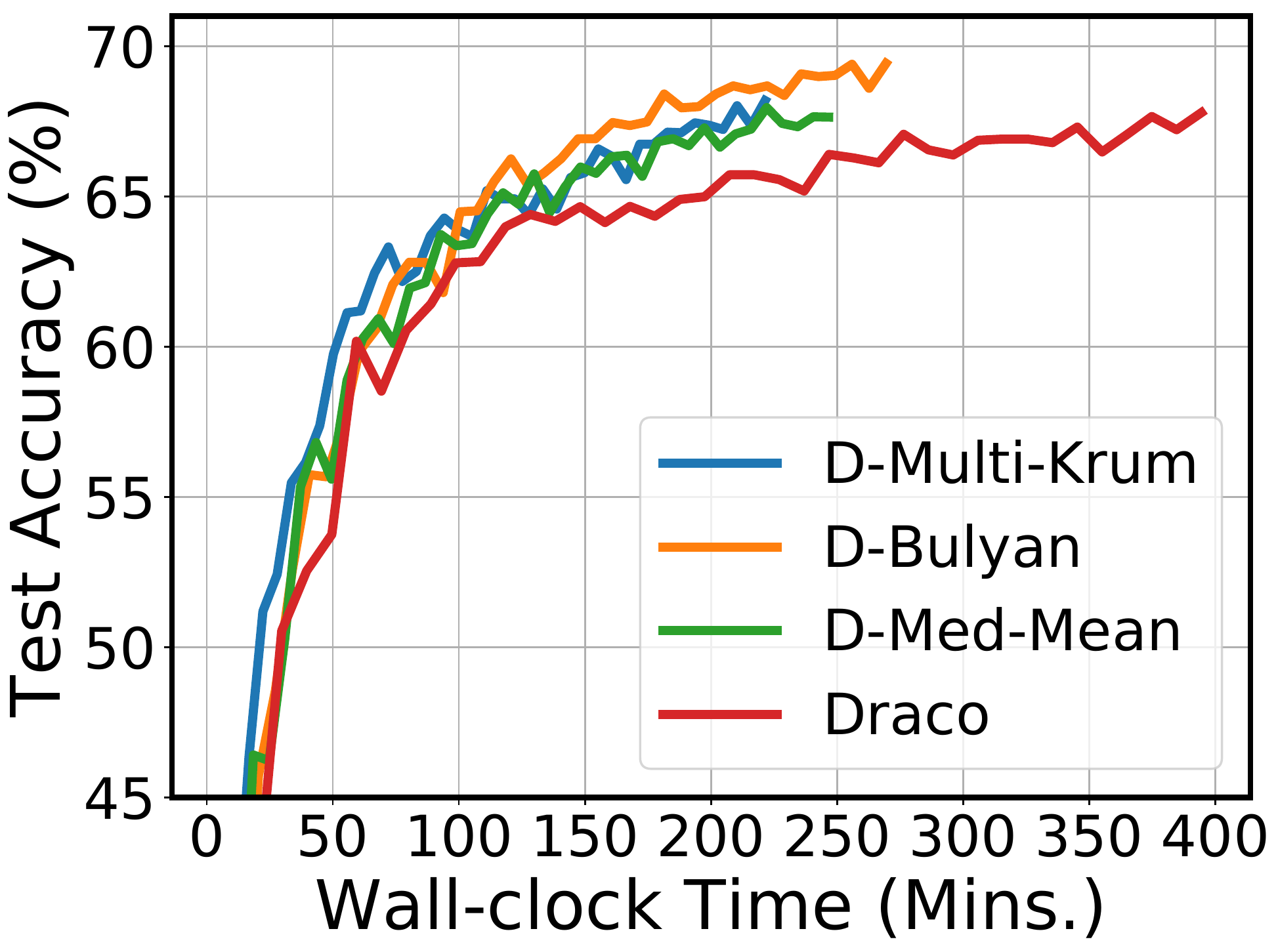}}
	\caption{Convergence with respect to runtime comparisons among \dracolite{} back-ended robust aggregation methods and \textsc{Draco} under \textit{reverse gradient} Byzantine attack on different dataset and model combinations: (a) ResNet-18 trained on CIFAR-10 dataset; (b) VGG13-BN trained on CIFAR-100 dataset}
	\label{fig:appendixCompareDracoConvergence}
\end{figure*}

\begin{figure*}[htp]
	\centering      
	\subfigure[ResNet-18 on CIFAR-10]{\includegraphics[width=0.4\textwidth]{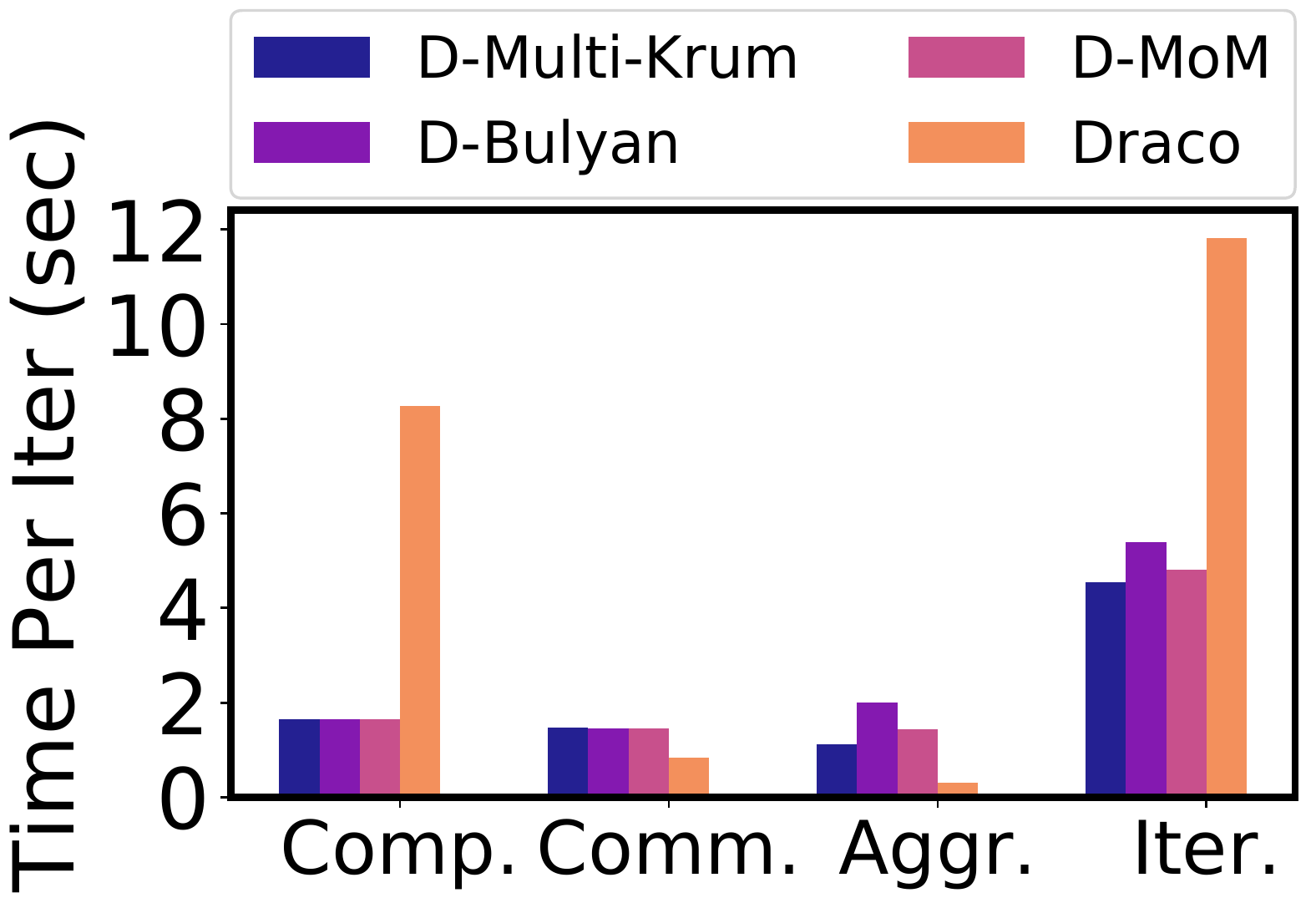}}
	\subfigure[VGG13-BN on CIFAR-100]{\includegraphics[width=0.4\textwidth]{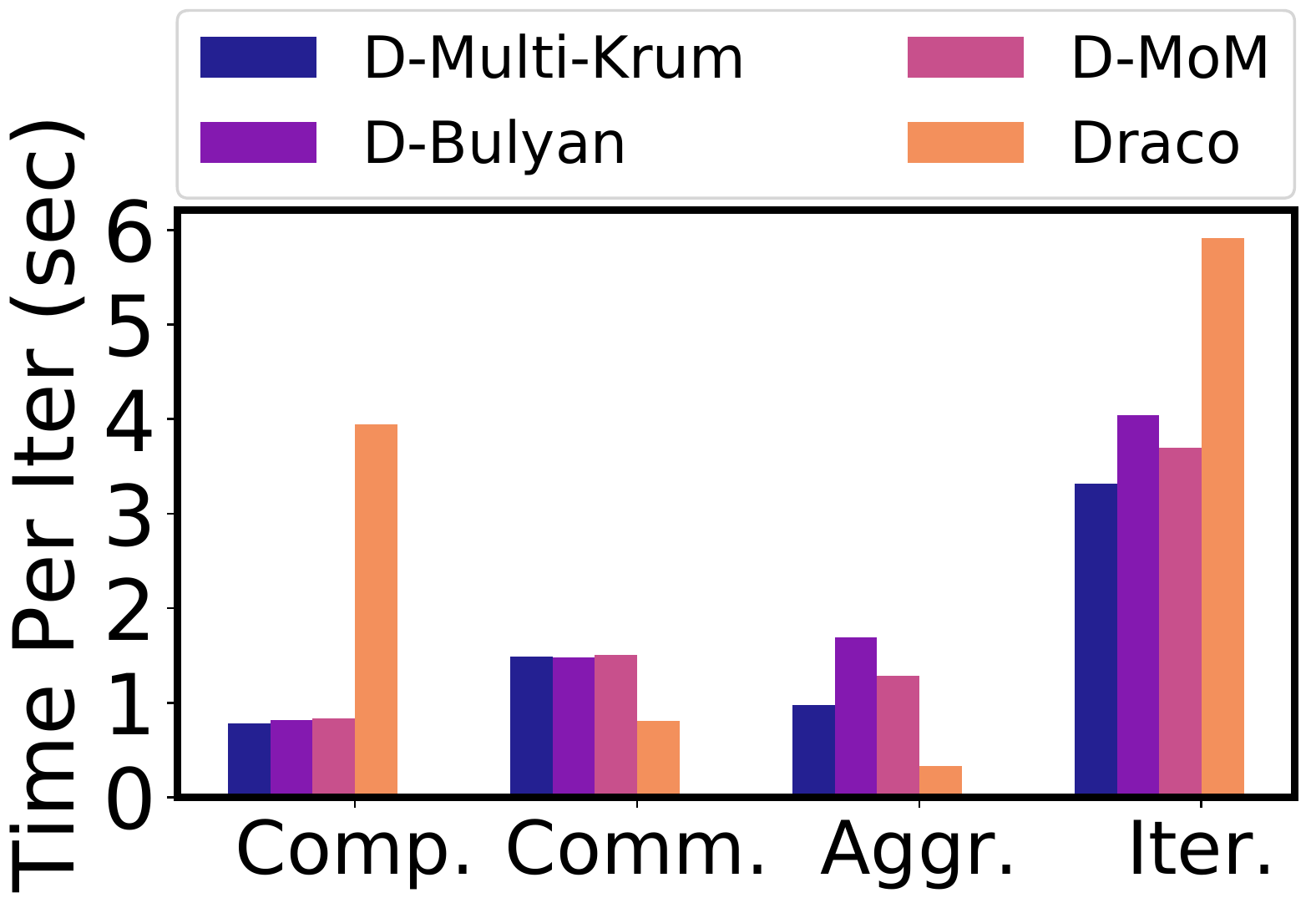}}
	\caption{Convergence with respect to runtime comparisons among \dracolite{} back-ended robust aggregation methods and \textsc{Draco} under \textit{reverse gradient} Byzantine attack on different dataset and model combinations: (a) ResNet-18 trained on CIFAR-10 dataset; (b) VGG13-BN trained on CIFAR-100 dataset}
	\label{fig:appendixCompareDracoIter}
\end{figure*}

\end{document}